\documentclass[runningheads]{llncs}
\pdfoutput=1
\usepackage{graphicx}
\usepackage{epstopdf}
\usepackage{comment}
\usepackage{amsmath,amssymb} % define this before the line numbering.
\usepackage{color}
\usepackage[lined,boxed,ruled, commentsnumbered]{algorithm2e}
\usepackage{subfigure}
\usepackage{wrapfig}
\usepackage{enumitem}
\usepackage{multirow}
\usepackage{xr}
\usepackage{wrapfig}

\usepackage{amsfonts}
\usepackage{amsmath}
\usepackage{amssymb}
\usepackage{dsfont}

\newcommand{\st}{\text{s.t. }}
\newcommand{\supp}{\text{supp}}
\usepackage{amsmath,amsfonts,bm}

\def\eqref#1{equation~\ref{#1}}
\def\1{\bm{1}}

\newcommand{\test}{\mathcal{D_{\mathrm{test}}}}

\def\vx{{\bm{x}}}

\def\evbeta{{\beta}}

\def\evk{{k}}

\DeclareMathAlphabet{\mathsfit}{\encodingdefault}{\sfdefault}{m}{sl}
\SetMathAlphabet{\mathsfit}{bold}{\encodingdefault}{\sfdefault}{bx}{n}

\def\gD{{\mathcal{D}}}

\def\gL{{\mathcal{L}}}
\def\gM{{\mathcal{M}}}

\def\gS{{\mathcal{S}}}
\def\gT{{\mathcal{T}}}

\def\gX{{\mathcal{X}}}
\def\gY{{\mathcal{Y}}}

\begin{document}
\pagestyle{headings}
\mainmatter
\def\ECCVSubNumber{3318}  % Insert your submission number here

\title{Meta-Learning with Network Pruning} % Replace with your title

\begin{comment}
\titlerunning{ECCV-20 submission ID \ECCVSubNumber}
\authorrunning{ECCV-20 submission ID \ECCVSubNumber}
\author{Anonymous ECCV submission}
\institute{Paper ID \ECCVSubNumber}
\end{comment}
%******************

%\begin{comment}
\titlerunning{Meta-Learning with Network Pruning}
\author{Hongduan Tian\inst{1} \and
Bo Liu\inst{2} \and
Xiao-Tong Yuan\inst{1} \and
Qingshan Liu\inst{1}}

\authorrunning{Tian et al.}
\institute{B-DAT Lab, Nanjing University of Information Science and Technology, Nanjing, 210044, China \and
  JD Finance America Corporation, Mountain View, CA 94043, USA\\
\email{\{hongduan.tian,kfliubo,xtyuan1980\}@gmail.com},
%\email{kfliubo@gmail.com},
%\email{xtyuan1980@gmail.com}\\
\email{qsliu@nuist.edu.cn}}
%\end{comment}
%******************
\maketitle
%\vspace{-1em}

\begin{abstract}
\emph{Meta-learning} is a powerful paradigm for few-shot learning. Although with remarkable success witnessed in many applications, the existing optimization based meta-learning models with over-parameterized neural networks have been evidenced to ovetfit on training tasks. To remedy this deficiency, we propose a network pruning based meta-learning approach for overfitting reduction via explicitly controlling the capacity of network. A uniform concentration analysis reveals the benefit of network capacity constraint for reducing generalization gap of the proposed meta-learner. We have implemented our approach on top of Reptile assembled with two network pruning routines: Dense-Sparse-Dense (DSD) and Iterative Hard Thresholding (IHT). Extensive experimental results on benchmark datasets with different over-parameterized deep networks demonstrate that our method not only effectively alleviates meta-overfitting but also in many cases improves the overall generalization performance when applied to few-shot classification tasks.

\end{abstract}

\section{Introduction}

The ability of adapting to a new task with several trials is essential for artificial agents. The goal of few-shot learning~\cite{santoro2016meta} is to build a model which is able to get the knack of a new task with limited training samples. Meta-learning~\cite{schmidhuber1987evolutionary,bengio1990learning,thrun2012learning} provides a principled way to cast few-shot learning as the problem of \emph{learning-to-learn}, which typically trains a hypothesis or learning algorithm to memorize the experience from previous tasks for a future task learning with very few samples. The practical importance of meta-learning has been witnessed in many vision and online/reinforcement learning applications including image classification~\cite{Ravi2016OptimizationAA,li2017meta}, multi-arm bandit~\cite{sung2017learning} and 2D navigation~\cite{finn2017model}.

Among others, one particularly simple yet successful meta-learning paradigm is first-order optimization based meta-learning which aims to train hypotheses that can quickly adapt to unseen tasks by performing one or a few steps of (stochastic) gradient descent~\cite{Ravi2016OptimizationAA,finn2017model}. Reasons for the recent increasing attention to this class of gradient-optimization based methods include their outstanding efficiency and scalability exhibited in practice~\cite{nichol2018first}.

\textbf{Challenge and motivation.} A challenge in the existing meta-learning approaches is their tendency to overfit~\cite{mishra2018a,yoon2018bayesian}. When training an over-parameterized meta-learner such as very deep and/or wide convolutional neural networks (CNNs) which are powerful for representation learning, there are two sources of potential overfitting at play: the inter-task overfitting of meta-learner (or\emph{ meta-overfitting}) to the training tasks and the inner-task overfitting of task-specific learner to the task training data. There have been recent efforts put to deal with inner-task overfitting~\cite{lee2019meta,zintgraf2019fast}. The study on the inter-task meta-overfitting, however, still remains under explored. Since in principle the optimization-based meta-learning is designed to learn fast from small amount of data in new tasks, we expect the meta-overfitting to play a more important role in influencing the overall generalization performance of the trained meta-learner.

Sparsity model is a promising tool for high-dimensional machine learning with guaranteed statistical efficiency and robustness to overfitting~\cite{maurer2012structured,yuan2018gradient,abramovich2018high}. It has been theoretically and numerically justified by~\cite{arora2018stronger} that sparsity benefits considerably the generalization performance of deep neural networks. In the regime of compact deep learning, the so called \emph{network pruning} technique has been widely studied and evidenced to work favorably in generating sparse subnetworks without compromising generalization performance~\cite{Frankle2019TheLT,jin2016training,han2016dsd}. Inspired by these remarkable success of sparsity models, it is natural to conjecture that sparsity would also be beneficial for enhancing the robustness of optimization based meta-learning to meta-overfitting.

\textbf{Our contribution.} In this paper, we present a novel gradient-based meta-learning approach with explicit network capacity constraint for overfitting reduction. The problem is formulated as learning a sparse meta-initialization network from training tasks such that in a new task the learned subnetwork can quickly converge to the optimal solution via gradient descent. The core idea is to reduce meta-overfitting by controlling the counts of the non-zero parameters in the meta-learner during the training phase. Theoretically, we have established a uniform generalization gap bound for the proposed sparse meta-learner showing the benefit of capacity constraint for improving its generalization performance. Practically, we have implemented our approach in a joint algorithmic framework of Reptile~\cite{nichol2018first} with network pruning, along with two instantiations using Dense-Sparse-Dense (DSD)~\cite{han2016dsd} and Iterative Hard Thresholding (IHT)~\cite{jin2016training} as network pruning subroutines, respectively. The actual performance of our approach has been extensively evaluated on few-shot classification tasks with over-parameterized wide CNNs. The obtained results demonstrate that our method can effectively alleviate overfitting and achieve similar or even superior generalization performance to the conventional dense models.

\section{Related Work}\label{related_work}
\label{gen_inst}

\textbf{Optimization-based meta-learning.}\quad The family of optimization-based meta-learning approaches usually learn a good hypothesis which can be fast adapted to unseen tasks~\cite{Ravi2016OptimizationAA,finn2017model,nichol2018first,khodak2019provable}. Compared to the metric~\cite{koch2015siamese,snell2017prototypical} and memory~\cite{weston2014memory,santoro2016meta} based meta-learning algorithms, optimization based meta-learning algorithms are gaining increasing attention due to their simplicity, versatility and effectiveness. As a recent leading framework for optimization-based meta-learning, MAML~\cite{finn2017model} is designed to estimate a meta-initialization network which can be well fine-tuned in an unseen task via only one or few steps of minibatch gradient descent. Although simple in principle, MAML requires computing Hessian-vector product for back-propagation, which could be computationally expensive when the model is big. The first-order MAML (FOMAML) is therefore proposed to improve  the computational efficiency by simply ignoring the second-order derivatives in MAML. Reptile~\cite{nichol2018first} is another approximated first-order algorithm which works favorably since it maximizes the inner product between gradients from the same task yet different minibatches, leading to improved model generalization. Recently, several hypothesis biased regularized meta-learning approaches have been studied in~\cite{denevi2019learning,khodak2019provable,zhou2019efficient} with provable strong generalization performance guarantees provided for convex problems. In~\cite{lee2019meta}, the meta-learner is treated as a feature embedding module of which the output is used as input to train a multi-class kernel support vector machine as base learner. To deal with overfitting, the CAVIA method~\cite{zintgraf2019fast} decomposes the meta-parameters into the so called context parameters and shared parameters. The context parameters are updated for task adaption with limited capacity while the shared parameters are meta-trained for generalization across tasks.

\textbf{Network pruning.}
Early network weight pruning algorithms date back to Optimal Brain Damage~\cite{lecun1990optimal} and Optimal Brain Surgeon~\cite{hassibi1993optimal}. A dense-to-sparse algorithm was developed by~\cite{han2015learning} to firstly remove near-zero weights and then fine tune the preserved weights. As a serial work of dense-to-sparse, the dense-sparse-dense (DSD) method~\cite{han2016dsd} was proposed to re-initialize the pruned parameters as zero and retrain the entire network after the dense-to-sparse pruning phase. The iterative hard thresholding (IHT) method~\cite{jin2016training} shares a similar spirit with DSD to conduct multiple rounds of iteration between pruning and retraining. ~\cite{srinivas2015data} proposed a data-free method to prune the neurons in a trained network. In~\cite{louizos2017learning}, an $L_{0}$-norm regularized risk minimization framework was proposed to learn sparse networks during training. More recently, ~\cite{Frankle2019TheLT} introduced and studied the ``lottery ticket hypothesis" which assumes that once a network is initialized, there should exist an optimal subnetwork, which can be learned by pruning, that performs as well as the original network or even superior.

Despite the remarkable success achieved by both meta-learning and network pruning, it still remains largely open to investigate the impact of network pruning on alleviating the meta-overfitting of optimization based meta-learning, which is of primal interest to our study in this paper.

\section{Method}
\label{headings}

%In this section, we present our method along with generalization analysis.

%approach that embeds the non-structured network pruning techniques in first-order meta-learning method in order to alleviate the overfitting resulted from over-parameterized models. We will present theoretical analysis to demonstrate that sparsity-regularization in non-structured network pruning can improve the generalization performance of first-order meta-learning approach.

\subsection{Problem Setup}

We consider the $\displaystyle N$-way $\displaystyle K$-shot problem as defined in~\cite{vinyals2016matching}. Tasks are sampled from a specific distribution $\displaystyle p(\gT)$ and will be divided into \emph{meta training set} $\displaystyle \gS^{tr}$, \emph{meta validation set} $\displaystyle \gS^{val}$, and \emph{meta testing set} $\displaystyle \gS^{test}$. Classes in different datasets are disjoint (i.e., the class in $\displaystyle \gS^{tr}$ will not appear in $\displaystyle \gS^{test}$). During training, each task is made up of support set $\displaystyle \gD^{supp}$ and query set $\displaystyle \gD^{query}$. Both $\displaystyle \gD^{supp}$ and $\displaystyle \gD^{query}$ are sampled from the same classes of $\displaystyle \gS^{tr}$. $\displaystyle \gD^{supp}$ is used for training while $\displaystyle \gD^{query}$ is used for evaluation. For a $\displaystyle N$-way $\displaystyle K$-shot classification task, we sample $\displaystyle N$  out of the $C$ classes from dataset, and then $\displaystyle K$ samples are sampled from each of these classes to form $\displaystyle \gD^{supp}$, namely $\displaystyle \gD^{supp}=\displaystyle \{(\vx_c^k, y_c^k), k=1,2,...,K; c=1,2,...,N\}$. For example, for a 5-way 2-shot task, we sample 2 data-label pairs from each of 5 classes, thus, such a task has 10 samples. Usually, several other samples of the same classes will be sampled to compose $\displaystyle \gD^{query}$. For example, $\gD^{query}$ is used in Reptile~\cite{nichol2018first} in evaluation steps. We use the loss function $\ell(v,y)$ to measure the discrepancy between the predicted score vector $v\in \mathbb{R}^C$ and the true label $y \in \{1,...,C\}$.

\noindent \textbf{Notation.} For an integer $n$, we denote $[n]$ as the abbreviation of the index set $\{1,...,n\}$. We use $\odot$ to denote the element-wise product operator. We say a function $g: \mathbb{R}^p \mapsto \mathbb{R}$ is $G$-Lipschitz continuous if $|g(\theta) - g(\theta')| \le  G\|\theta -\theta'\|_2$, and $g$ is $H$-smooth if it obeys  $\|\nabla g(\theta)-\nabla g(\theta')\|_2 \leq  H\|\theta - \theta'\|_2$.

\subsection{Meta-Learning with Model Capacity Constraint}

Our ultimate goal is to learn a good initialization of parameters for a convolutional neural network $f_{\theta}: \gX \mapsto \gY$, where $\theta$ is the model parameters set, from a set of training tasks such that the learned initialization network generalizes well to future unseen tasks. Inspired by the recent remarkable success of MAML~\cite{finn2017model} and the strong generalization capability of sparse deep learning models~\cite{Frankle2019TheLT,arora2018stronger}, during sparse(or network pruning) phase, we propose to learn from previous task experience a sparse subnetwork started from which the future task-specific networks can be efficiently learned using first-order optimization methods. To this end, we introduce the following layer-wise sparsity constrained stochastic first-order meta-learning formulation:
\begin{equation}\label{prob:population}
\resizebox{.9\hsize}{!}{$\min\limits_{\theta} \mathcal{R}(\theta):= \mathbb{E}_{T \sim p(\gT)} \left[\gL_{\gD^{query}_{T}}\left(\theta - \eta \nabla_{\theta}\gL_{\gD^{supp}_{T}}(\theta)\right)\right], \ \st \|\theta_l\|_0\leq\evk_l, \ l\in [L],$}
\end{equation}
where $\gL_{\gD^{supp}_{T}}(\theta)=\frac{1}{NK}\sum_{(\vx_c^k,y_c^k)\in \gD^{supp}_{T}}\ell(f_\theta(\vx_c^k),y_c^k)$ is the empirical risk for task $T$ and $\gL_{\gD^{query}_{T}}(\theta)$ is similarly defined as the loss evaluated over the query set and $\eta$ is the learning rate. In the constraint, $||\theta_l||_0$ denotes the number of non-zero entries in the parameters of $l$-th layer $\theta_l$ which is required to be no larger than a user-specified sparsity level $k_l$, and $\displaystyle L$ is the total number of network layers.

In general, the mathematical formulation of task distribution $p(\gT)$ is unknown but we usually have access to a set of i.i.d. training tasks $S=\{T_i\}_{i=1}^M$ sampled from $p(\gT)$. Thus the following empirical version of the population form in~\eqref{prob:population} is alternatively considered for training:
\begin{equation}\label{prob:empirical}
\resizebox{.9\hsize}{!}{$\min\limits_{\theta} \mathcal{R}_S(\theta ):= \frac{1}{M}\sum_{i=1}^M \left[\gL_{\gD^{query}_{T_i}}\left(\theta - \eta \nabla_{\theta}\gL_{\gD^{supp}_{T_i}}(\theta)\right)\right], \ \st \|\theta_l\|_0\leq\evk_l, \ l\in [L].$}
\end{equation}
To compare with MAML, our model shares an identical objective function, but with the layer-wise sparsity constraints $\|\theta_l\|_0\le k_l$ imposed for the purpose of enhancing learnability of the over-parameterized meta-initialization network. In view of the ``lottery ticket hypothesis''~\cite{Frankle2019TheLT}, the model in~\eqref{prob:empirical} can be interpreted as a first-order meta-learner for estimating a subnetwork, or a ``\emph{winning ticket}'', for future task learning. Inspired by the strong statistical efficiency and generalization guarantees of sparsity models~\cite{yuan2018gradient,arora2018stronger}, we will very shortly show that such a subnetwork is able to achieve advantageous generalization performance over the dense initialization networks learned by vanilla MAML.

\subsection{Generalization Analysis}

We provide in this section a task-level generalization performance analysis for the proposed model in~\eqref{prob:empirical}. Let $p$ be the total number of parameters in the over-parameterized network and $\Theta \subseteq \mathbb{R}^p$ be the domain of interest for $\theta$. Let $k=\sum_{l=1}^L k_l$ be the total desired sparsity level of the subnetwork. The following uniform concentration bound is our main result.

\begin{theorem}\label{thrm:universe_generalizaion}
Assume that the domain of interest $\Theta $ is bounded by $R$ and the loss function $\ell(f_\theta(\vx),y)$ is $G$-Lipschitz continuous and $H$-smooth with respect to $\theta$. Suppose that $0\le\ell(f_\theta(\vx),y) \le B$ for all pairs $\{f_\theta(\vx), y\}$. Then for any $\delta\in (0,1)$, with probability at least $1-\delta$ over the random draw of $S$, the generalization gap is uniformly upper bounded for all $\theta$ satisfying $\|\theta_l\|_0\le k_l, l\in [L]$ as
\[
\begin{aligned}
\left|\mathcal{R}(\theta)  - \mathcal{R}_S(\theta)\right| \le \mathcal{O}\left(B\sqrt{\frac{k\log(p\sqrt{M}GR(1+\eta H)/(Bk)) + \log(1/\delta) }{M}}\right).
\end{aligned}
\]
\end{theorem}
In comparison to the $\mathcal{O}\left(\sqrt{p/M}\right)$ uniform bound established in Lemma~\ref{lemma:universe_support} (see Appendix~\ref{append:proof_thrm_1}) for dense networks, the uniform bound established in Theorem~\ref{thrm:universe_generalizaion} is substantially stronger when $k\ll p$, which shows the benefit of network capacity constraint for generalization.

Specially for margin-based multiclass classification, let us consider the margin operator $\mathcal{M}(v,y):=\max_j[v]_j - [v]_y$ associated with the score prediction vector $v\in \mathbb{R}^C$ and label $y\in\{1,...,C\}$. Let $\ell_\gamma (f_\theta(\vx),y)=h_\gamma(\mathcal{M}(f_\theta(\vx),y))$ be a \emph{surrogate} loss of the binary loss (i.e., $\mathds{1}[y\neq \arg\max_{j}[f_\theta(\vx)]_j]$) defined with respect to proper $\gamma$-margin based loss $h_\gamma$ such as the hinge/ramp losses and their smoothed variants~\cite{pillutla2018smoother}. By definition, we must have $\mathds{1}[y\neq \arg\max_{j}[f_\theta(\vx)]_j]\le \ell_\gamma (f_\theta(\vx),y)$. In this case, we denote $\mathcal{R}_{\gamma,S}$ the meta-training risk with loss function $\ell_\gamma$ and $\mathcal{\tilde R}_{\gamma}$ the corresponding population risk in which the task-level query loss $\gL_{\gD^{query}_{T}}$ is evaluated using binary loss as classification error. Then as a direct consequence of Theorem~\ref{thrm:universe_generalizaion}, we can establish the following result for margin-based prediction.

\begin{corollary}\label{corol:universe_generalizaion_multiclass}
Suppose that the margin-based loss $\ell_\gamma$ is used for model training. Then under the conditions in Theorem~\ref{thrm:universe_generalizaion}, for any $\delta\in (0,1)$, with probability at least $1-\delta$ the following bound holds for all $\theta$ satisfying $\|\theta_l\|_0\le k_l, l\in [L]$:
\[
\begin{aligned}
\mathcal{\tilde R}_\gamma(\theta) \le \mathcal{R}_{\gamma,S}(\theta) + \mathcal{O}\left(B\sqrt{\frac{k\log(p\sqrt{M}GR(1+\eta H)/(Bk)) + \log(1/\delta) }{M}}\right).
\end{aligned}
\]
\end{corollary}
\begin{remark}
%We comment that the above $\mathcal{O}(\sqrt{k/M})$ margin bound derived in the context of sparse meta-learning can be readily extended to sparse deep nets training. To compare with the margin bound in~\cite{bartlett2017spectrally} for dense networks which essentially scales as $\mathcal{O}(a^L/(\gamma \sqrt{M}))$ with some $a>0$ controlling the spectrum norm of the layer-wise parameter matrix, our bound is tighter when $k \ll a^{2L}/\gamma^2$. Also, Corollary~\ref{corol:universe_generalizaion_multiclass} can be easily generalized for arbitrary convex surrogates (e.g., cross-entropy loss) of binary loss under proper regularity conditions.
We comment that the above $\mathcal{O}(\sqrt{k/M})$ margin bound derived in the context of sparse meta-learning can be readily extended to sparse deep nets training. Also, the bound can be easily generalized for arbitrary convex surrogates (e.g., cross-entropy loss) of binary loss under proper regularity conditions.
\end{remark}

\section{Algorithm}

We have implemented the proposed model in Equation~\ref{prob:empirical} based on Reptile~\cite{nichol2018first} (see Algorithm~\ref{alg:reptile}) which is a scalable method for optimization-based meta-learning in form of Equation~\ref{prob:empirical} but without layer-wise sparsity constraint. In order to handle the sparsity constraint, we follow the principles behind the widely applied \emph{dense-sparse-dense} (DSD)~\cite{han2016dsd} and \emph{iterative hard thresholding} (IHT)~\cite{jin2016training} network pruning algorithms to alternate the Reptile iteration between pruning insignificant weights in each layer and retraining the pruned network.

\subsection{Main Algorithm: Reptile with Iterative Network Pruning}

The algorithm of our network-pruning-based Reptile method is outlined in Algorithm~\ref{alg:reptile_dsd}. The learning procedure contains a pre-training phase followed by an iterative procedure of network pruning and retraining. We would like to stress that since our ultimate goal is not to do network compression, but to reduce meta-overfitting via controlling the sparsity level of the meta-initialization network, the final output of our algorithm is typically dense after the retraining phase, which has been evidenced in practice to be effective for improving the generalization performance during testing phase. In the following subsections, we describe the key components of our algorithm in details.
\begin{algorithm}[tb]
\caption{Reptile with Iterative Network Pruning}
\label{alg:reptile_dsd}
\SetKwInOut{Input}{Input}\SetKwInOut{Output}{Output}\SetKw{Initialization}{Initialization}
\Input{inner loop learning rate $\eta$, outer loop learning rate $\beta$, layer-wise sparsity level $\{k_l\}_{l=1}^L$, mini-batch batch size $s$ for meta training.}
\Output{$\theta^{(t)}$.}
\Initialization{Randomly initialize $\theta^{(0)}$.}

\tcc{\underline{\textbf{Pre-training with Reptile}}}

\While{the termination condition is not met}{
 $\theta^{(0)}=Reptile(\theta^{(0)}, \eta, \beta, s)$;
}
%\tcc{Iterative network pruning and retraining}
\For{$t=1, 2, ...$}{
\tcc{\underline{\textbf{Pruning phase}}}

Generate a network zero-one mask $\gM^{(t)}$ whose non-zero entries at each layer $l$ are those top $k_l$ entries in $\theta^{(t)}_l$;

Compute $\theta^{(t)}_{\mathcal{M}}=\theta^{(t)}\odot \gM^{(t)}$;

\tcc{Subnetwork fine-tune with Reptile}
\While{the termination condition is not met}{
 $\theta^{(t)}=Reptile(\theta^{(t)}_{\mathcal{M}}, \eta, \beta, s)$;\\
% Compute $\theta^{(t)}_{\mathcal{M}}=\theta^{(t)}\odot \gM^{(t)}$;
}
\tcc{\underline{\textbf{Retraining phase}}}

%Set $\theta^{(t)}\leftarrow \theta^{(t)}_{\mathcal{M}}$;

\While{the termination condition is not met}{
 $\theta^{(t)}=Reptile(\theta^{(t)}, \eta, \beta, s)$;
}
}
\end{algorithm}
%\vspace{-1em}
%\vspace{-1.5em}

\subsubsection{Model Pretraining}

For model pre-training, we run a few number of Reptile iteration rounds to generate a relatively good initialization. In each loop of the Reptile iteration, we first sample a mini-batch of meta-tasks $\{T_i\}_{i=1}^s$ from the task distribution $p(\mathcal{T})$. Then for each task $T_i$, we compute the adapted parameters via (stochastic) gradient descent as $\tilde \theta_{T_i} = \theta^{(0)} - \eta \nabla_{\theta}\gL_{\gD^{supp}_{T_i}}(\theta^{(0)})$, where $\tilde \theta_{T_i} $ denotes the task-specific parameters learned from each task $T_i$, $\theta^{(0)}$ is the current initialization of model parameters, $\eta$ is the inner-task learning rate, and $\gD^{supp}_{T_i}$ denotes the support set of task $T_i$. When all the task-specific parameters are updated, the initialization parameters will be updated according to $\theta^{(0)} = \theta^{(0)} + \evbeta \left(\frac{1}{s}\sum_{i=1}^s \tilde \theta_{T_i} - \theta^{(0)}\right)$ with learning rate $\beta$. Here we follow Reptile to use $\frac{1}{s}\sum_{i=1}^s \tilde \theta_{T_i} - \theta^{(0)}$ as an approximation to the negative meta-gradient, which has been evidenced to be effective for scaling up the MAML-type first-order meta-learning models~\cite{nichol2018first}.

%\vspace{-1.5em}
\begin{algorithm}[h]\caption{Reptile Algorithm~\cite{nichol2018first}}
\label{alg:reptile}
\SetKwInOut{Input}{Input}\SetKwInOut{Output}{Output}
%\SetKw{Initialization}{Initialization}
\Input{model parameters $\phi$, inner loop learning rate $\eta$, outer loop learning rate $\beta$, mini-batch batch size $s$ for meta training.}
\Output{the updated $\phi$.}
%\Initialization{Randomly initialize $\theta^{(0)}$.}

Sample a mini-batch tasks $\{T_i\}_{i=1}^s$ of size $s$;

For each task $T_i$, compute the task-specific adapted parameters using gradient descent:
\[
\tilde \phi_{T_i} = \phi - \eta \nabla_{\phi}\gL_{\gD^{supp}_{T_i}}(\phi);
\]
Update the parameters: $\phi = \phi + \evbeta \left(\frac{1}{s}\sum_{i=1}^s \tilde \phi_{T_i} - \phi\right)$.
\end{algorithm}
%\vspace{-1.5em}

\subsubsection{Iterative Network Pruning and Retraining}

After model pre-training, we proceed to the main loop of our Algorithm~\ref{alg:reptile_dsd} that carries out iterative network pruning and retraining.

\textbf{Pruning phase.} In this phase, we first greedily truncate out of the model a portion of near-zero parameters which are unlikely to contribute significantly to the model performance. To do so, we generate a network binary mask $\gM^{(t)}$ whose non-zero entries at each layer $l$ are those top $k_l$ (in magnitude) entries in $\theta^{(t)}_l$, and compute $\theta^{(t)}_{\mathcal{M}}=\theta^{(t)}\odot \gM^{(t)}$ as the sparsity restriction of $\theta^{(t)}$. Then we fine-tune the subnetwork over the mask $\gM^{(t)}$ by applying Reptile restrictively to this subnetwork with initialization $\theta^{(t)}_{\mathcal{M}}$. Our numerical experience suggests that sufficient steps of subnetwork fine-tuning tends to substantially improve the stability and convergence behavior of the method.

The fine-tuned subnetwork $\theta^{(t)}_{\mathcal{M}}$ at the end of the pruning phase is expected to reduce the chance of overfitting to noisy data. However, it is also believed that such subnetwork will reduce the capacity of the network, which could in turn lead to potentially biased learning with higher training loss. To remedy this issue, inspired by the retraining trick introduced in~\cite{han2016dsd} for network pruning, we propose to restore the pruned weights that would be beneficial for enhancing the model representation power to improve the overall generalization performance.

\textbf{Retraining phase.} In this phase, the layer-wise sparsity constraints are removed and the pruned parameters are re-activated for fine-tuning. The retraining procedure is almost identical to the pre-training phase, but with the main difference that the former is initialized with the subnetwork generated by the pruning phase while the latter uses random initialization. Such a retraining operation restores the representation capacity of the pruned parameters, which tends to lead to improved generalization performance in practice.
For theoretical justification, roughly speaking, since the sparse meta-initialization network obtained in the pruning phase generalizes well in light of Theorem~\ref{thrm:universe_generalizaion}, it is expected to serve as a good initialization for future retraining via gradient descent. Then according to the stability theory of gradient descent methods~\cite{hardt2016train}, the output dense network will also generalize well if the retraining phase converges quickly.

\subsection{Two Substantialized Implementations}

\noindent \textbf{Reptile with DSD pruning.} The DSD method is an effective network  pruning approach for preventing the learned model from capturing noise during the training~\cite{han2016dsd}. By implementing the main loop with $t=1$, the proposed Algorithm~\ref{alg:reptile_dsd} reduces to a DSD-based Reptile method for first-order meta-learning.

\noindent \textbf{Reptile with IHT pruning.} The IHT method~\cite{jin2016training} is another representative network pruning approach which shares a similar dense-sparse-dense spirit with DSD. Different from the one-shot weight pruning and network training by DSD, IHT is designed to perform multiple rounds of iteration between pruning and retraining, and hence is expected to have better chance to find an optimal sparse subnetwork than DSD does. By implementing the main loop of Algorithm~\ref{alg:reptile_dsd} with $t>1$, we actually obtain a variant of Reptile with IHT-type network pruning.

\vspace{-0.2em}
\section{Experiments}\label{experiment}
\vspace{-0.5em}
In this section, we carry out a numerical study for algorithm performance evaluation aiming to answer the following three questions empirically: (Q1) Section~\ref{ssect:classification}: \emph{Does our method contribute to improve the generalization performance?} (Q2) Section~\ref{ssect:ablation}: \emph{What roles do pre-training phase and retraining phase play in our method?} (Q3) Section~\ref{ssect:complex}: \emph{Can our method work on more complex models?}

\subsection{Few-Shot Classification Performances}\label{ssect:classification}

We first evaluate the prediction performance of our method for few-shot classification tasks on two popular benchmark datasets: MiniImageNet~\cite{vinyals2016matching} and TieredImageNet~\cite{ren2018meta}. We have also evaluated our method on Omniglot~\cite{lake2011one} with numerical results relegated to Appendix~\ref{Appdx:Omniglot_complete_results} due to space limit. The network used in our experiments is consistent with that considered for Reptile\cite{nichol2018first}. We test with varying channel number $\{32, 64, 128, 256\}$ in each convolution layer to show the robustness of our algorithms to meta-overfitting. See Appendix~\ref{Appdx:Detailed_Settings} for more details about Model, datasets and hyperparameters.

\subsubsection{MiniImageNet}\label{miniimagenet_exp}
%\vspace{-1em}
%\textbf{Experimental Settings.}
The MiniImageNet dataset consists of 64 training classes, 12 validation classes and 24 test classes. For DSD-based Reptile, with $32$ channels, we set the iteration numbers for the pre-traning, pruning and retraining phases respectively as $3\times10^4$, $5\times 10^4$ and $2\times 10^4$, while with $64, 128, 256$ channels, the corresponding number is $3\times 10^4$, $6\times 10^4$ and $10^4$ respectively. For IHT-based Reptile model training, we first pre-train the model for $2\times 10^4$ iterations. Then we iterate between the sparse model fine-tuning (with $1.5\times 10^4$ iterations) and dense-model retraining (with $5\times 10^3$ iterations) for $t=4$ rounds. The setting of other model training related parameters is identical to those in~\cite{nichol2018first}.
%We set the number of outer iterations as 100K.
\begin{table}[h]
\caption{Results on MiniImageNet under varying number of channels and pruning rates.
}
\label{MiniImageNet5wayresults_table_article}
\begin{center}
%\vspace{-1em}
%\begin{small}
\begin{tabular}{|l|l|l|cc|}
\hline
\bf Methods  &\bf Backbone  &\bf Rate   &\bf 5-way 1-shot   &\bf 5-way 5-shot\\
\hline
\multirow{4}*{Reptile baseline}  &32-32-32-32 &0$\%$                     &50.30$\pm$0.40$\%$                &64.27$\pm$0.44$\%$ \\
                                ~&64-64-64-64 &0$\%$                     &51.08$\pm$0.44$\%$                &65.46$\pm$0.43$\%$ \\
                                ~&128-128-128-128 &0$\%$                 &49.96$\pm$0.45$\%$                &64.40$\pm$0.43$\%$ \\
                                ~&256-256-256-256 &0$\%$                 &48.60$\pm$0.44$\%$                &63.24$\pm$0.43$\%$ \\
\hline
\multirow{3}*{CAVIA baseline}    &32-32-32-32    &0$\%$                  &47.24$\pm$0.65$\%$                &59.05$\pm$0.54$\%$ \\
                                ~&128-128-128-128   &0$\%$                   &49.84$\pm$0.68$\%$                &64.63$\pm$0.54$\%$ \\
                                ~&512-512-512-512   &0$\%$               &51.82$\pm$0.65$\%$                &65.85$\pm$0.55$\%$ \\
\hline
\multirow{4}*{DSD+Reptile}       &32-32-32-32   &$40\%$                  &\bf50.83$\pm$0.45$\%$             &\bf65.24$\pm$0.44$\%$ \\
                                ~&64-64-64-64   &$30\%$                  &\bf 51.91$\pm$0.45$\%$            &\bf67.23$\pm$0.43$\%$ \\
                                ~&128-128-128-128 &$50\%$                &\bf 52.08$\pm$0.45$\%$            &\bf68.87$\pm$0.42$\%$ \\
                                ~&256-256-256-256 &$60\%$                &\bf53.00$\pm$0.45$\%$             &\bf68.04$\pm$0.42$\%$ \\
\hline
\multirow{4}*{IHT+Reptile}       &32-32-32-32  &$20\%$                   &50.26$\pm$0.47$\%$                &63.63$\pm$0.45$\%$ \\
                                ~&64-64-64-64  &$40\%$                   &\bf 52.59$\pm$0.45$\%$            &\bf 67.41$\pm$0.43$\%$ \\
                                ~&128-128-128-128  &$40\%$               &\bf52.73$\pm$0.45$\%$             &\bf68.69$\pm$0.42$\%$ \\
                                ~&256-256-256-256  &$60\%$               &49.85$\pm$0.44$\%$                &66.56$\pm$0.42$\%$ \\
\hline
\end{tabular}
%\end{small}
\end{center}
\vspace{-1em}
\end{table}

\textbf{Results.} The experimental results are presented in Table \ref{MiniImageNet5wayresults_table_article} and some additional results are provided in Table~\ref{MiniImageNet5wayresults_table} in Appendix~\ref{apdx_c}.
%%%%%%%%%%%%%%%%%%%%%%%%%%%%%%% total %%%%%%%%%%%%%%%%%%%%%%%%%%%%%%%%%%%%
From these results, we can observe that our methods consistently outperform the considered baselines.
%%%%%%%%%%%%%%%%%%%%%%%%%%%%%%% 32-channel %%%%%%%%%%%%%%%%%%%%%%%%%%%%%%%
The key observations are highlighted below:
\begin{itemize}[leftmargin=*]
  \item In the 32-channel setting in which the model is less prone to overfit, when applying DSD-based Reptile with $40\%$ pruning rate, the accuracy gain is $0.5\%$ on 5-way 1-shot tasks and $1\%$ on 5-way 5-shot tasks. In the 64-channel setting, our IHT-based Reptile approach respectively improves about $1.5\%$ and $1.95\%$ over the baselines on 5-way 1-shot tasks and 5-way 5-shot tasks. In the setting of 128-channel, the accuracy of DSD-based Reptile on 5-way 1-shot tasks is nearly $3\%$ higher than the baseline while on 5-way 5-shot tasks the gain is about $4.47\%$.
  \item In the 128-channel setting, the accuracies of our approaches are all higher than those of CAVIA, and the highest gain of accuracy is over $3\%$.
\end{itemize}

\begin{wrapfigure}{r}{0.6\textwidth}
%\vspace{-3em}
    \includegraphics[width=0.38\textheight]{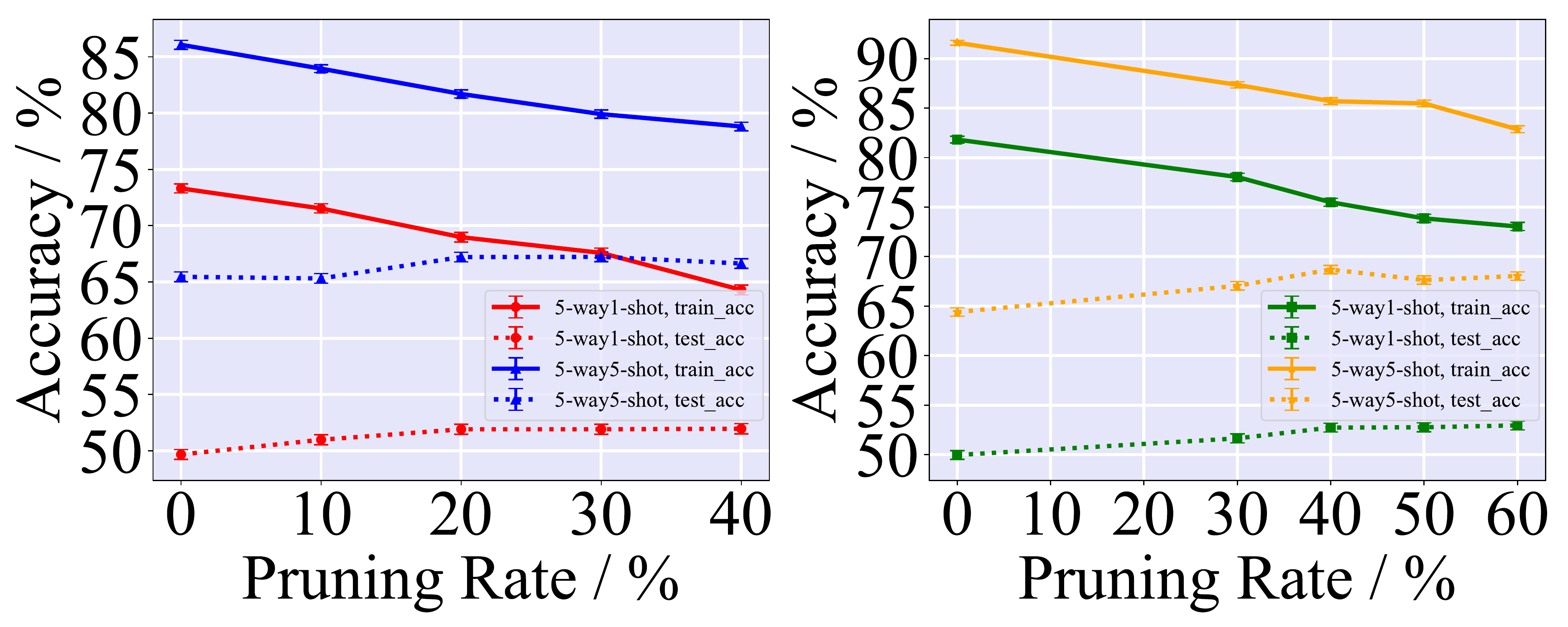}
    %\vspace{-2em}
	\caption{The generalization performance of DSD-based and IHT-based Reptile on 5-way 1-shot and 5-way 5-shot tasks under 64-channel settings. \emph{\bf Left: }DSD-based Reptile; \emph{\bf Right: }IHT-based Reptile.}
\label{Fig:Generalization}
%\vspace{-2.2em}
\end{wrapfigure}

It is also worth noting from Table \ref{MiniImageNet5wayresults_table_article} that  the accuracy of our algorithms tends to increase as the channel size increases while the baselines behave oppositely. Although in 256-channel case the performance of IHT-based approach drops compared with the 128-channel setting, it still achieves $\sim1.2\%$ accuracy gain over the baseline on 5-way 1-shot tasks and $\sim3.32\%$ on 5-way 5-shot tasks. These results clearly confirm the robustness of our algorithms to the meta-overfitting suffered from the over-parameterization of CNNs.

Figure~\ref{Fig:Generalization} shows the evolving curves of training and testing accuracy under varying pruning rates from $0$ to $40\%$ for DSD and IHT based Reptile. From these curves we can clearly observe that the gap between training accuracy and testing accuracy reduces when the pruning rate increases, which confirms the predictions of Theorem~\ref{thrm:universe_generalizaion} about the impact of network capacity on generalization.

\subsubsection{TieredImageNet}\label{tieredimagenet_exp}
%\textbf{Experimental Settings.}
\vspace{-1em}
The TieredImageNet dataset consists of 351 training classes, 97 validation classes and 160 test classes. For TieredImageNet dataset \cite{ren2018meta}, in DSD-based Reptile case, we set iteration numbers for the pre-training, pruning and retraining phase respectively as $3\times10^4$, $5\times 10^4$ and $2\times 10^4$ for all cases. In IHT-based Reptile, the values of iteration number are the same as those used in the previous experiments for MiniImageNet.
\begin{table}
%\vspace{-1.5em}
\caption{Results on TieredImageNet under varying number of channels and pruning rates.}
\label{TieredImageNet5wayresults_table_article}
\renewcommand{\arraystretch}{1.1}
\begin{center}
\vspace{-1em}
%\begin{small}
\begin{tabular}{|l|l|l|cc|}
\hline
\bf Methods  &\bf Backbone  &\bf Rate  &\bf 5-way 1-shot   &\bf 5-way 5-shot\\
\hline
\multirow{4}*{Reptile baseline} &32-32-32-32        &0$\%$           &50.52$\pm$0.45$\%$                &64.63$\pm$0.44$\%$ \\
                               ~&64-64-64-64        &0$\%$           &51.98$\pm$0.45$\%$                &67.70$\pm$0.43$\%$ \\
                               ~&128-128-128-128    &0$\%$           &53.30$\pm$0.45$\%$                &69.29$\pm$0.42$\%$ \\
                               ~&256-256-256-256    &0$\%$           &54.62$\pm$0.45$\%$                &68.06$\pm$0.42$\%$ \\
\hline

\multirow{4}*{DSD+Reptile}      &32-32-32-32        &10$\%$          &\bf50.94$\pm$0.46$\%$             &64.65$\pm$0.44$\%$ \\
                               ~&64-64-64-64        &10$\%$          &\bf 52.62$\pm$0.46$\%$            &66.69$\pm$0.43$\%$ \\
                               ~&128-128-128-128    &$10\%$          &53.39$\pm$0.46$\%$                &67.22$\pm$0.43$\%$ \\
                               ~&256-256-256-256    &$20\%$          &\bf 54.98$\pm$0.45$\%$            &67.98$\pm$0.43$\%$ \\
\hline
\multirow{4}*{IHT+Reptile}      &32-32-32-32        &$10\%$          &50.58$\pm$0.46$\%$                &63.09$\pm$0.45$\%$ \\
                               ~&64-64-64-64        &$20\%$          &\bf53.22$\pm$0.46$\%$             &66.15$\pm$0.44$\%$ \\
                               ~&128-128-128-128    &$10\%$          &\bf53.48$\pm$0.45$\%$             &\bf69.39$\pm$0.42$\%$ \\
                               ~&256-256-256-256    &$10\%$          &\bf55.06$\pm$0.45$\%$             &67.60$\pm$0.43$\%$ \\

\hline
\end{tabular}
%\end{small}
\end{center}
\vspace{-2.5em}
\end{table}

\textbf{Results.}\quad The experimental results are partly presented in Table~\ref{TieredImageNet5wayresults_table_article}. More experimental results are available in Table~\ref{TieredImageNet5wayresults_table} in Appendix~\ref{apdx_c}. %From the table, we can observe that accuracies of baselines increase with the number of channels increases, which is different from those in Section ~\ref{miniimagenet_exp}. We think that the reason is TieredImageNet dataset \cite{ren2018meta} includes more image classes and overfitting is not so evident as that on MiniImageNet dataset.
In 5-way 1-shot classification tasks, both DSD-besed Reptile approach and IHT-based Reptile approach outperform the baselines in all cases. In 32-channel setting, with DSD-based Reptile approach, the improvement of accuracy is $0.42\%$ compared with baseline. In 64-channel setting, the accuracies of DSD-based Reptile and IHT-based Reptile respectively achieve $0.64\%$ and $1.24\%$ improvements. And for 256-channel setting, the best performance is also $0.44\%$ better than baseline.

However, in most 5-way 5-shot classification tasks, the performance of our method drops. We conjecture that the reason is TieredImageNet dataset, compared with MiniImageNet dataset, contains more classes.

\subsection{On the Impact of Hyperparameters}\label{ssect:ablation}

We next conduct a set of experiments on MiniImageNet to better understand the impact of pre-training and dense retraining on the task-specific testing performance.
\begin{table}
%\vspace{-1em}
\caption{Results of ablation study in the 5-way setting. The ``$\pm$'' shows $95\%$ confidence intervals, the "P.T" means "Pre-traning" and the "R.T" means "Retraining".}
\label{ablation}
\begin{center}
\begin{tabular}{|l|c|c|cc|}
\hline
 \bf Methods   &\bf P.T&\bf R.T &\bf 5-way 1-shot    &\bf 5-way 5-shot
\\ \hline
Reptile baseline(64)       &-           &-            &51.08$\pm$0.44$\%$        &65.46$\pm$0.43$\%$ \\
Reptile baseline(128)      &-           &-            &49.96$\pm$0.45$\%$        &64.40$\pm$0.43$\%$ \\

\hline
DSD+Reptile(64, 40$\%$)    &$\surd$     &$\times$     &43.92$\pm$0.43$\%$        &60.09$\pm$0.45$\%$ \\
DSD+Reptile(128, 60$\%$)   &$\surd$     &$\times$     &47.06$\pm$0.44$\%$        &55.07$\pm$0.44$\%$ \\
IHT+Reptile(64, 40$\%$)    &$\surd$     &$\times$     &40.03$\pm$0.41$\%$        &60.59$\pm$0.45$\%$ \\
IHT+Reptile(128, 60$\%$)   &$\surd$     &$\times$     &42.01$\pm$0.42$\%$        &52.71$\pm$0.45$\%$ \\
\hline
% no pretrain
DSD+Reptile(64, 40$\%$)    &$\times$    &$\surd$      &50.84$\pm$0.45$\%$        &66.32$\pm$0.44$\%$ \\
DSD+Reptile(128, 60$\%$)   &$\times$    &$\surd$      &51.04$\pm$0.45$\%$        &67.23$\pm$0.44$\%$ \\
IHT+Reptile(64, 40$\%$)    &$\times$    &$\surd$      &52.07$\pm$0.45$\%$        &66.90$\pm$0.43$\%$ \\
IHT+Reptile(128, 60$\%$)   &$\times$    &$\surd$      &52.58$\pm$0.45$\%$        &67.83$\pm$0.42$\%$ \\
\hline
DSD+Reptile(64, $40\%$)    &$\surd$     &$\surd$      &51.96$\pm$0.45$\%$        &66.64$\pm$0.43$\%$ \\
DSD+Reptile(128, $60\%$)   &$\surd$     &$\surd$      &52.27$\pm$0.45$\%$        &\bf68.44$\pm$0.42$\%$ \\
IHT+Reptile(64, $40\%$)    &$\surd$     &$\surd$      &\bf52.59$\pm$0.45$\%$     &\bf67.41$\pm$0.43$\%$ \\
IHT+Reptile(128, $60\%$)   &$\surd$     &$\surd$      &\bf52.95$\pm$0.45$\%$     &68.04$\pm$0.42$\%$ \\
\hline
\end{tabular}
\end{center}
\vspace{-3em}
\end{table}

%\subsubsection{Effect of Pre-training and Retraining Phase}
We begin by performing ablation study on pre-training and retraining phases. Each time only one of them is removed from our method. For fair comparison, other settings are the same as proposed in Section \ref{miniimagenet_exp}. This study is conducted for both DSD- and IHT-based Reptile approaches. The results of the experiments are listed in Table \ref{ablation}.

\textbf{Impact of the Retraining phase.} It can be clearly seen from group of results in Table~\ref{ablation} that the retraining phase plays an important role in the accuracy performance of our method. Under the same pruning rate, without the retraining phase, the accuracy of both DSD-based and IHT-based Reptile approach drops dramatically. For an instance, in the 64-channel case with $40\%$ pruning rate, the variant of IHT-based Reptile without retraining phase suffers from a $\sim 11\%$ drop in accuracy compared with the baseline. On the other side, as shown in Figure \ref{Fig:Ablation_gap} that sparsity structure of the network does help to reduce the gap between training accuracy and testing accuracy even without the retraining phase. This confirms the benefit of sparsity for generalization gap reduction as revealed by Theorem~\ref{thrm:universe_generalizaion}. Therefore, the network pruning phase makes the model robust to overfitting but in the meanwhile tend to suffer from the deteriorated training loss. The retraining phase helps to restore the capacity of the model to further improve the overall generalization performance.
\begin{wrapfigure}{r}{0.5\textwidth}
\vspace{-0.5em}
    \includegraphics[width=0.3\textheight]{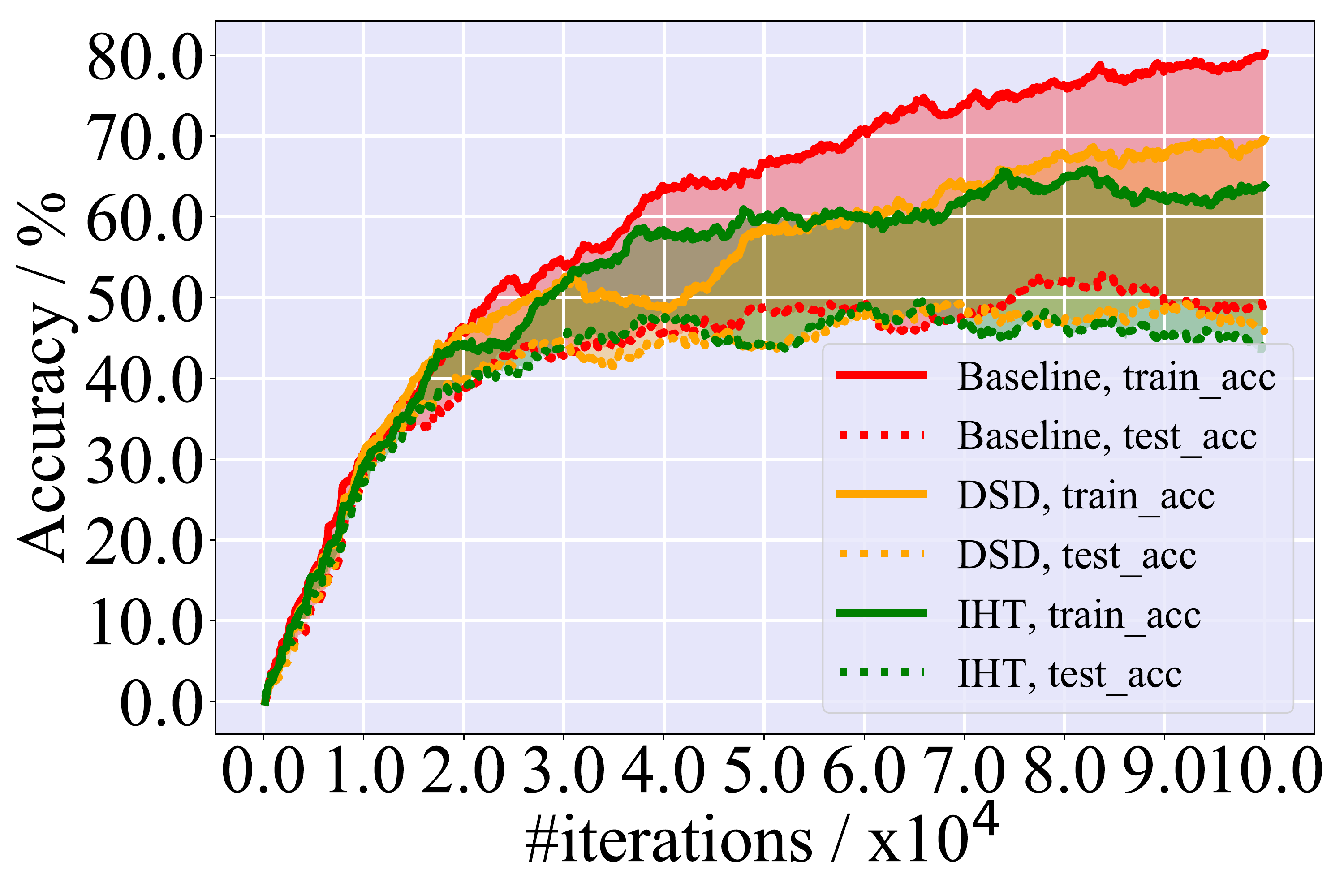}
    \vspace{-1em}
	\caption{Ablation study on retraining phase for both DSD-based Reptile and IHT-based Reptile on 64-channel case. The gap between the training accuracy and test accuracy of the variant algorithm of our method becomes smaller than that of baseline.}
\label{Fig:Ablation_gap}
\vspace{-1.5em}
\end{wrapfigure}

\textbf{Impact of the Pre-training phase.} From Table~\ref{ablation}, we observe that without pre-training phase, the variant algorithms still outperform baselines. Such results demonstrate the importance of pruning and retraining phase from another perspective that merely pruning and retraining the over-parameterized models can achieve similar empirical performance to our method. However, the variant algorithms fail to outperform our method. Since in network pruning, pre-training phase is treated as a necessary phase used to find a set of model parameters which is important~\cite{Frankle2019TheLT,han2016dsd}, we conjecture that it is the prematurely pruning before the model being well trained that leads to the drop of the performance.

%\subsubsection{More Quantitative Results}
%\vspace{0.5cm}
We now perform experiments to further show how performance varies with different hyperparameters. The tested hyperparameters include (1)The number of pre-training and retraining iterations in DSD-based Reptile; (2) the number of iterations in an IHT pruning-retraining interval; (3) the ratio of pruning iterations in an interval. To be clear, we define $ratio=(Iter_{prune}/Iter_{interval})\%$. In experiments above, we set $Iter_{prune}=1.5\times10^4$ in a 20000-iteration IHT interval, which means the ratio is 75$\%$.
\begin{figure}[h]
%\vspace{-1em}
\begin{center}
\centering
%\vspace{-1em}
    \subfigure[Study on pre-training iterations]{
        \begin{minipage}[t]{0.486\linewidth}\label{SubFig:DSD_hyper_pretrain}
        %\vspace{-0.4cm}
        \includegraphics[width=0.32\textheight]{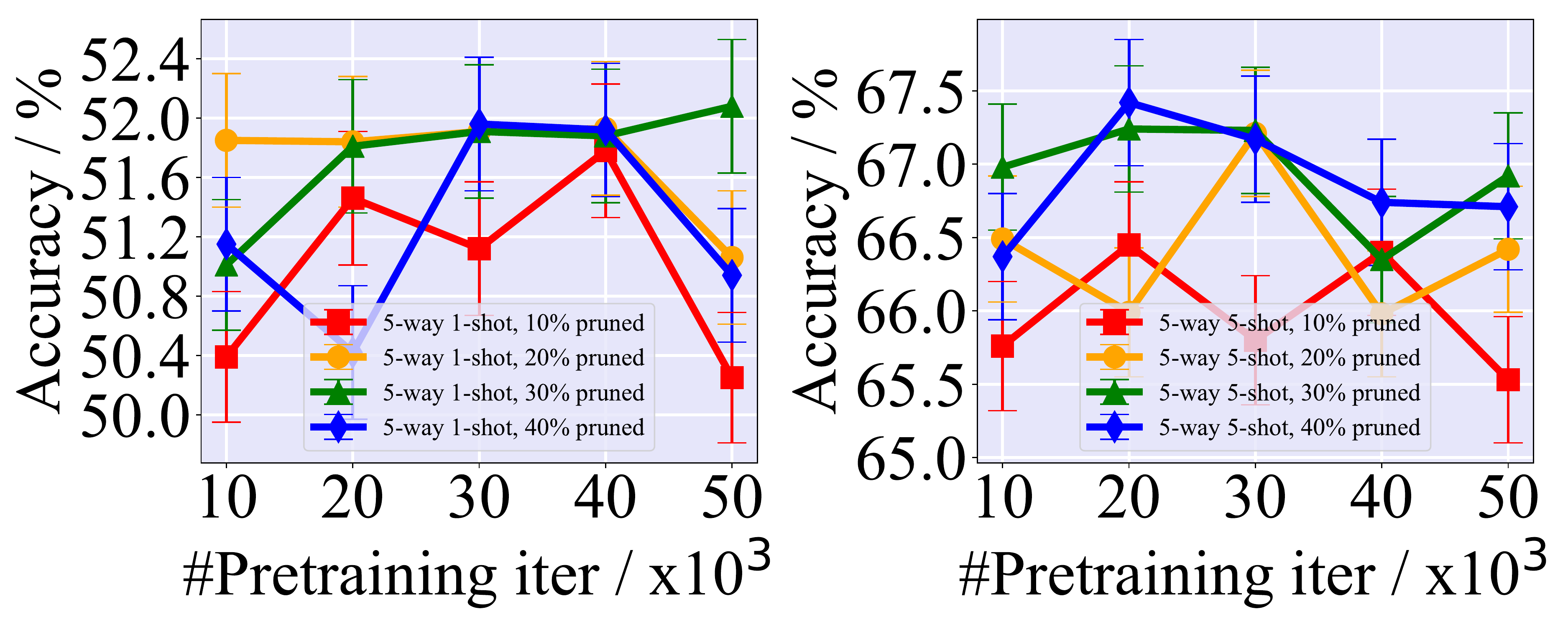}
        \end{minipage}}
    \subfigure[Study on retraining iterations]{
        \begin{minipage}[t]{0.48\linewidth}\label{SubFig:DSD_hyper_retrain}
        %\vspace{-0.4cm}
        \includegraphics[width=0.32\textheight]{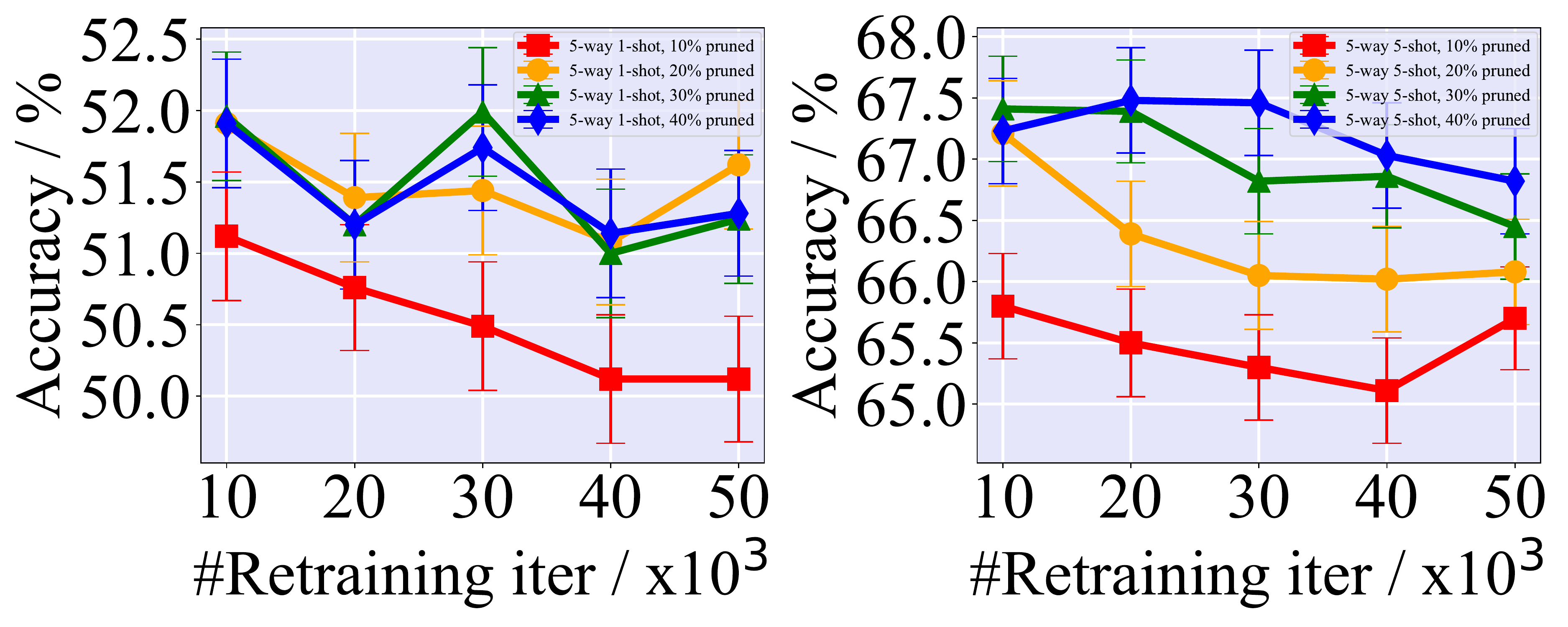}
        \end{minipage}}
\vspace{-1em}
\caption{Study of hyperparameters of DSD-based Reptile. (a). Study on the number of pre-straining iterations. (b). Study on the number of retraining iterations.}
\label{Fig:DSD_hyper}
%\vspace{-1em}
\end{center}
\end{figure}
%\vspace{-3em}

\begin{figure}[h]
\begin{center}
\centering
%\vspace{-0.6em}
	\subfigure[Study on ratio of pruning iterations]{
        \begin{minipage}[t]{0.486\linewidth}\label{SubFig:IHT_hyper_ratio}
        \includegraphics[width=0.32\textheight]{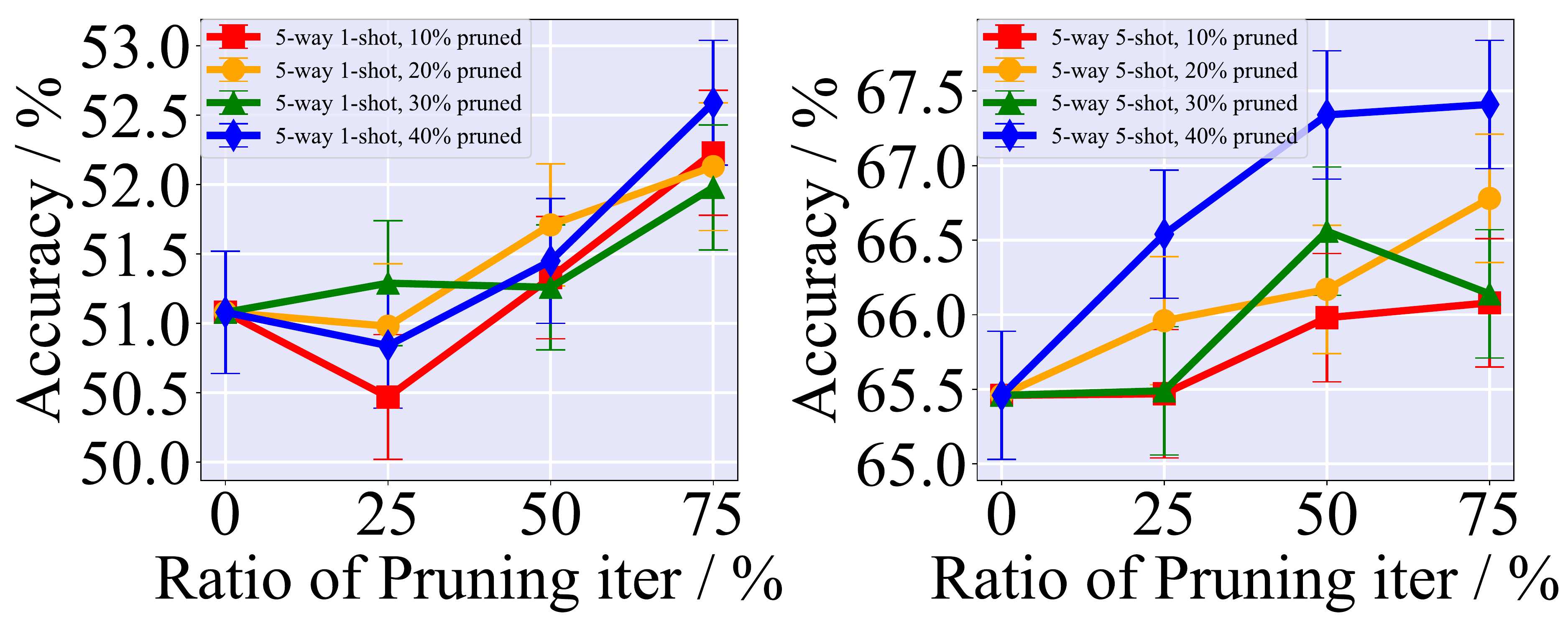}
        \end{minipage}}
    \subfigure[Study on the interval iteration number]{
        \begin{minipage}[t]{0.48\linewidth}\label{SubFig:IHT_hyper_interval}
        \includegraphics[width=0.32\textheight]{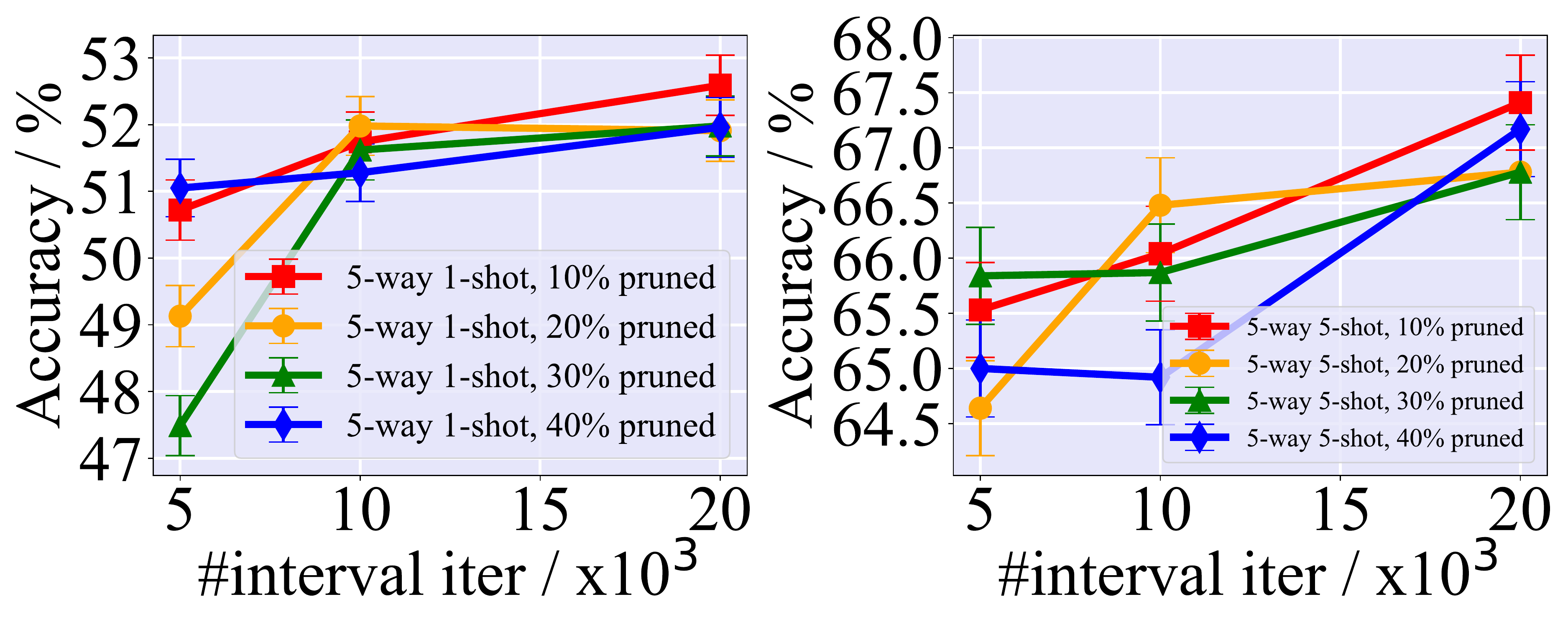}
        \end{minipage}}
\vspace{-1em}
\caption{Study on hyperparameters of IHT-based Reptile. (a). Study on ratio of pruning iterations. (b). Study on the number of interval iterations.}
\label{Fig:IHT_hyper}
\end{center}
\end{figure}
%\vspace{-1em}

\emph{\bf{DSD-based Reptile.}} Figure~\ref{Fig:DSD_hyper} manifests the performance of DSD-based Reptile varying with the two hyperparameters, number of pre-training and retraining iterations. Figure~\ref{SubFig:DSD_hyper_pretrain} reveals that for most cases, too much or too little pre-training will both lead to the deterioration of performance. This is consistent with our ablation study that pre-training helps find a set of robust sparse parameters that is important and excessive pre-training, which reduces the iterations of pruning phase, undermines the generalization performance. Figure~\ref{SubFig:DSD_hyper_retrain} shows that better performance can be obtained when retraining iterations are smaller than 30K, which indicates that only a small number of retraining steps are required to restore the accuracy without overfitting again.

\emph{\bf IHT-based Reptile.} Figure~\ref{Fig:IHT_hyper} shows the hyper-parameter sensitivity results of IHT-based Reptile. Figure~\ref{SubFig:IHT_hyper_ratio} shows the performance under different ratio of pruning iterations in an IHT interval. It's clear that better performance can be obtained when the ratio is larger than $50\%$, which means pruning iterations are more than retraining iterations. This reveals that more pruning iterations are required to alleviate overfitting and a small number of retraining steps are enough to help compensate the loss of accuracy. Figure~\ref{SubFig:IHT_hyper_interval} shows the performance under varying  number of iterations in an IHT interval. We can see that with the interval iterations increasing from 5K to 20K, the accuracies get improved. This suggests that sufficient steps are required to train a robust model in a loop of pruning-retraining.

\begin{table}[h]
%\vspace{-2em}
\caption{Results on complex Networks. The ``$\pm$'' shows $95\%$ confidence intervals.}
\label{Table_complex_nets}
\begin{center}
\vspace{-1em}
\setlength{\tabcolsep}{1.5mm}{
\begin{tabular}{|l|c|c|cc|}
\hline
 \bf Methods                              &\bf Backbone       &\bf Rate     &\bf 5-way 1-shot           &\bf 5-way 5-shot
\\ \hline
MetaOptNet(rerun)~\cite{lee2019meta}      &ResNet-12          &-            &61.95$\pm$0.60$\%$         &77.79$\pm$0.45$\%$ \\
%\hdshline
DSD-MetaOptNet                            &ResNet-12          &20$\%$       &\bf62.16$\pm$0.62$\%$      &77.51$\pm$0.48$\%$ \\ \hline
CAVIA(rerun)~\cite{zintgraf2019fast}      &128-128-128-128    &-            &48.76$\pm$0.99$\%$         &62.54$\pm$0.78$\%$ \\
%\hdashrule
DSD-CAVIA                                 &128-128-128-128    &10$\%$       &\bf49.53$\pm$0.93$\%$      &\bf63.34$\pm$0.79$\%$ \\
IHT-CAVIA                                 &128-128-128-128    &10$\%$       &\bf49.97$\pm$0.97$\%$      &\bf63.48$\pm$0.78$\%$ \\
\hline
\end{tabular}}
\end{center}
%\vspace{-1em}
\end{table}
\subsection{Performance on More Complex Networks}\label{ssect:complex}
We further implement our method on \emph{MetaOptNet}~\cite{lee2019meta} and \emph{CAVIA}~\cite{zintgraf2019fast} to evaluate its the performance on more complex network structures. For MetaOpeNet, we select the ResNet-12 as the network, and SVM as the head and the dropout in the network is replaced by our method. We respectively set the number of iterations of pre-training, pruning and retraining as 5 epochs, 20 epochs and 15 epochs. The learning rate is 0.1 in the first 30 epochs, 0.006 in next 5 epochs and 0.0012 in the final 5 epochs. The dataset used is MiniImageNet. Since our experiments are conducted on 4 RTX 2080Ti GPUs(11GB) while MetaOptNet is trained on 4 Titan X GPUs(12GB), we have to reduce the training shots from 15 to 10 in our experiments. For fair comparison, we rerun the baseline on the same model as in that paper with 10 training shots.

For CAVIA, we apply our method directly on the network parameters, and the context parameters will not be pruned. In DSD-based CAVIA case, the numbers of the iterations for pre-training, pruning and retraining phase are respectively 20K, 20K and 20K. In IHT-based CAVIA case, the iteration number of pre-training is 20K, and the iterative phase include 2 sparse-dense processes. Each sparse-dense process contains 20K iterations in which 16K iterations are for pruning fine-tuning and 4K iterations are for dense retraining. Other settings are the same as those in paper~\cite{zintgraf2019fast}.

As shown in Table~\ref{Table_complex_nets}, for MetaOptNet, our method gains $0.2\%$ improvement on 5-way 1-shot tasks compared with baseline. On 5-way 5-shot tasks, our method can still obtain similar performance. However, there is a trade-off between accuracy and training time. For CAVIA, all cases outperform the baselines, which shows the strong power of our methods in alleviating the overfitting. Overall, our method can to some extent improve the generalization performance even in the complex models when facing with scarce data.

\section{Conclusion}

In this paper, we proposed a cardinality-constrained meta-learning approach for improving generalization performance via explicitly controlling the capacity of over-parameterized neural networks. We have theoretically proved that the generalization gap bounds of the sparse meta-learner have polynomial dependence on the sparsity level rather than the number of parameters. Our approach has been implemented in a scalable meta-learning framework of Reptile with the sparsity level of parameters maintained by network pruning routines including dense-sparse-dense and iterative hard thresholding. Extensive experimental results on benchmark few-shot classification tasks, along with hyperparameter impact study and study on complex networks, confirm our theoretical predictions and demonstrate the power of network pruning and retraining for improving the generalization performance of gradient-optimization-based meta-learning.

\section*{Acknowledgements}

Xiao-Tong Yuan is supported in part by National Major Project of China for New Generation of AI under Grant No.2018AAA0100400 and in part by Natural Science Foundation of China (NSFC) under Grant No.61876090 and No.61936005. Qingshan Liu is supported by NSFC under Grant No.61532009 and No.61825601.

%\clearpage

\bibliographystyle{splncs04}
\bibliography{mybib}

\clearpage
\appendix
\section{Proofs of Results}
\subsection{Proof of Theorem~\ref{thrm:universe_generalizaion}}\label{append:proof_thrm_1}
We need the following lemma which guarantees the uniform convergence of $\mathcal{R}_S(\theta)$ towards $\mathcal{R}(\theta)$ for all $\theta$ when the loss function is Lipschitz continuous and smooth, and the optimization is limited on a bounded domain.
\begin{lemma}\label{lemma:universe_support}
Assume that the domain of interest $\Theta \subseteq \mathbb{R}^p$ is bounded by $R$ and the loss function $\ell(f_\theta(\vx),y)$ is $G$-Lipschitz continuous and $H$-smooth with respect to $\theta$. Also assume that $0\le\ell(f_\theta(\vx),y) \le B$ for all $\{f_\theta(\vx), y\}$. Then for any $\delta\in(0,1)$, the following bound holds with probability at least $1-\delta$ over the random draw of sample set $S$ for all $\theta \in \Theta$,
\[
\left|\mathcal{R}(\theta) - \mathcal{R}_S(\theta)\right| \le  \mathcal{O}\left(B\sqrt{\frac{\log(1/\delta) + p\log(\sqrt{M}G R(1+\eta H)/B)}{M}}\right).
\]
\end{lemma}
\begin{proof}
For any task $T$, let us denote $\tilde\ell(\theta; T):=\gL_{\gD^{query}_{T}}\left(\theta - \eta \nabla_{\theta}\gL_{\gD^{supp}_{T}}(\theta)\right)$. Since $\ell(f_\theta(\vx),y)$ is $G$-Lipschitz continuous with respect to $\theta$, we can show that
\[
\begin{aligned}
|\tilde\ell(\theta;T) - \tilde\ell(\theta';T)| \le& G\|\theta - \eta \nabla_{\theta}\gL_{\gD^{supp}_{T}}(\theta) - \theta' + \eta \nabla_{\theta}\gL_{\gD^{supp}_{T}}(\theta')\| \\
\le& G\left(\|\theta - \theta'\| + \eta \|\nabla_{\theta}\gL_{\gD^{supp}_{T}}(\theta) - \nabla_{\theta}\gL_{\gD^{supp}_{T}}(\theta')\|\right) \\
\le& G(1+\eta H)\|\theta - \theta'\|,
\end{aligned}
\]
which indicates that $\tilde\ell(\theta;T)$ is $G(1+\eta H)$-Lipschitz continuous for any task $T$.

As a subset of an $L_2$-sphere, it is standard that the covering number of $\Theta$ with respect to the $L_2$-distance is upper bounded by
\[
\mathcal{N}(\epsilon, \Theta, L_2) \le \mathcal{O}\left(\left(1+\frac{R}{\epsilon}\right)^p\right).
\]
Since the task-level loss function $\tilde\ell(\theta;T)$ is $G(1+\eta H)$-Lipschitz continuous as shown above, it can be verified that the covering number of the class of functions $\mathcal{\tilde L}=\left\{T \mapsto \tilde\ell(\theta;T)\mid \theta \in \Theta\right\}$ with respect to $L_\infty$-distance $L_\infty(\tilde\ell(\theta_1;\cdot),\tilde\ell(\theta_2;\cdot)):=\sup_{T} |\tilde\ell(\theta_1;T) - \tilde\ell(\theta_2;T)|$ is given by
\[
\mathcal{N}(\epsilon, \mathcal{\tilde L}, L_\infty) \le \mathcal{N}\left(\frac{\epsilon}{G(1+\eta H)}, \Theta, L_2\right) \le \mathcal{O}\left(\left(1+\frac{GR(1+\eta H)}{\epsilon}\right)^p\right).
\]
Therefore, there exists a set of points $\Omega \subseteq \mathbb{R}^p$ with cardinality at most $\mathcal{N}(\epsilon, \mathcal{\tilde L}, L_\infty)$ such that the following bound holds for any $\theta\in \Theta$:
\[
\min_{\omega\in \Omega}|\tilde\ell(\theta;T)-\tilde\ell(\omega;T)| \le \epsilon, \ \forall T.
\]
For an arbitrary $\omega\in \Omega$, based on Hoeffding��s inequality (note that $\ell(\cdot,\cdot)\le B$ implies $\tilde \ell(\cdot,\cdot)\le B$) we have
\[
\mathbb{P}\left(|\mathcal{R}_S(\omega) - \mathcal{R}(\omega)|>t\right) \le \exp\left\{-\frac{Mt^2}{2B^2}\right\}.
\]
For any $\theta\in \Theta$, based on triangle inequality we can show that there exits $\omega_\theta\in \Omega$ such that
\[
\begin{aligned}
|\mathcal{R}_S(\theta) - \mathcal{R}(\theta)| =& |\mathcal{R}_S(\theta) - \mathcal{R}_S(\omega_\theta) + \mathcal{R}_S(\omega_\theta) - \mathcal{R}(\omega_\theta) + \mathcal{R}(\omega_\theta) - \mathcal{R}(\theta)| \\
\le& 2\epsilon + |\mathcal{R}_S(\omega_\theta) -\mathcal{R}(\omega_\theta)| \le 2\epsilon + \max_{\omega\in \Omega}|\mathcal{R}_S(\omega) - \mathcal{R}(\omega)|.
\end{aligned}
\]
Applying uniform bound we know that
\[
\begin{aligned}
&\mathbb{P}\left(\sup_{\theta\in \Theta}|\mathcal{R}(\theta)- \mathcal{R}_S(\theta)|\ge 2\epsilon + t \right) \\
\le& \mathcal{N}(\epsilon, \mathcal{L}, \ell_\infty)\exp\left(-\frac{M t^2}{2B^2}\right)\le \mathcal{O}\left(\left(1 + \frac{GR(1+\eta H)}{\epsilon}\right)^p\exp\left(-\frac{M t^2}{2B^2}\right)\right).
\end{aligned}
\]
Let us choose $\epsilon=B/\sqrt{M}$ and
\[
t = \sqrt{2}B\sqrt{\frac{\log(1/\delta) + p\log(GR(1+\eta H)/\epsilon)}{M}}
\]
such that the right hand side of the previous inequality equals $\delta$. Then we obtain that with probability at least $1-\delta$
\[
\begin{aligned}
\sup_{\theta \in \Theta}|\mathcal{R}(\theta) - \mathcal{R}_S(\theta)| \le \mathcal{O}\left(B\sqrt{\frac{\log(1/\delta) + p\log(\sqrt{M}GR(1+\eta H)/B)}{M}}\right).
\end{aligned}
\]
This proves the desired result.
\end{proof}
Based on this lemma, we can readily prove the main result in the theorem.
\begin{proof}[Proof of Theorem~\ref{thrm:universe_generalizaion}]
For any fixed supporting set $J \in \mathcal{J}$, by applying Lemma~\ref{lemma:universe_support} we obtain that the following uniform convergence bound holds for all $\theta$ with $\supp(\theta) \subseteq J$ with probability at least $1-\delta$ over $S$:
\[
\left|\mathcal{R}(\theta) - \mathcal{R}_S(\theta)\right| \le \mathcal{O}\left(B\sqrt{\frac{\log(1/\delta) + k\log(\sqrt{M}GR(1+\eta H)/B)}{M}}\right).
\]
Since by constraint the parameter vector $\theta$ is always $k$-sparse, we thus have $\supp(\theta) \in \mathcal{J}$. Then by union probability we get that with probability at least $1-\delta$, the following bound holds for all $\theta$ with $\|\theta\|_0\le k$:
\[
\left|\mathcal{R}(\theta)  - \mathcal{R}_S(\theta)\right| \le  \mathcal{O}\left(B\sqrt{\frac{\log(|\mathcal{J}|) + \log(1/\delta) + k\log(\sqrt{M}GR(1+\eta H)/B)}{M}}\right).
\]
It remains to bound the cardinality $|\mathcal{J}|$. From~\cite[Lemma 2.7]{rigollet201518} we know $|\mathcal{J}|=\binom{p}{k}\le \left(\frac{ep}{k}\right)^k$, which then implies the desired generalization gap bound. This completes the proof.
\end{proof}

\subsection{Proof of Corollary~\ref{corol:universe_generalizaion_multiclass}}

\begin{proof}
Let $\mathcal{R}_\gamma$ be a population version of $\mathcal{R}_{\gamma,S}$ with margin-based loss function $\ell_\gamma$ used for computing both $\gL_{\gD^{supp}_{T}}$ and $\gL_{\gD^{query}_{T}}$. Since $\ell_\gamma$ is a surrogate of the binary loss as used by $\mathcal{\tilde R}$ for query classification error evaluation, we must have $\mathcal{\tilde R}\le \mathcal{R}_\gamma$. Then the desired bound follows directly by invoking Theorem~\ref{thrm:universe_generalizaion} to the considered margin loss.
\end{proof}

%\newpage
\section{Detailed Experimental Settings}\label{Appdx:Detailed_Settings}
\subsection{Model}\label{appdx:model}

The model used in our experiments is consistent with that considered for Reptile\cite{nichol2018first}. The model used throughout the experiment contains 4 sequential modules. Each module contains a convolutional layer with 3$\times$3 kernel, followed by a batch normalization and a ReLU activation. Additionally for the experiments on MiniImageNet, a $2\times2$ max-pooling pooling is used on the batch normalization layer output while for Omniglot a stride of 2 is used in convolution. The above network structure design is consistent with those considered for Reptile in~\cite{nichol2018first}. We test with varying channel number $\{32, 64, 128, 256\}$ in each convolution layer to show the robustness of our algorithms to meta-overfitting.

\subsection{Datasets}\label{subsec:dataset}
There are three popular benchmark datasets used in our experiments.

\begin{figure}
\begin{center}
\centering
\vspace{-2em}
    \subfigure[5-way 1-shot tasks generated from Omniglot]{
        \begin{minipage}[t]{0.48\linewidth}\label{omniglot_tasks}
        %\vspace{-0.4cm}
        \includegraphics[width=0.31\textheight]{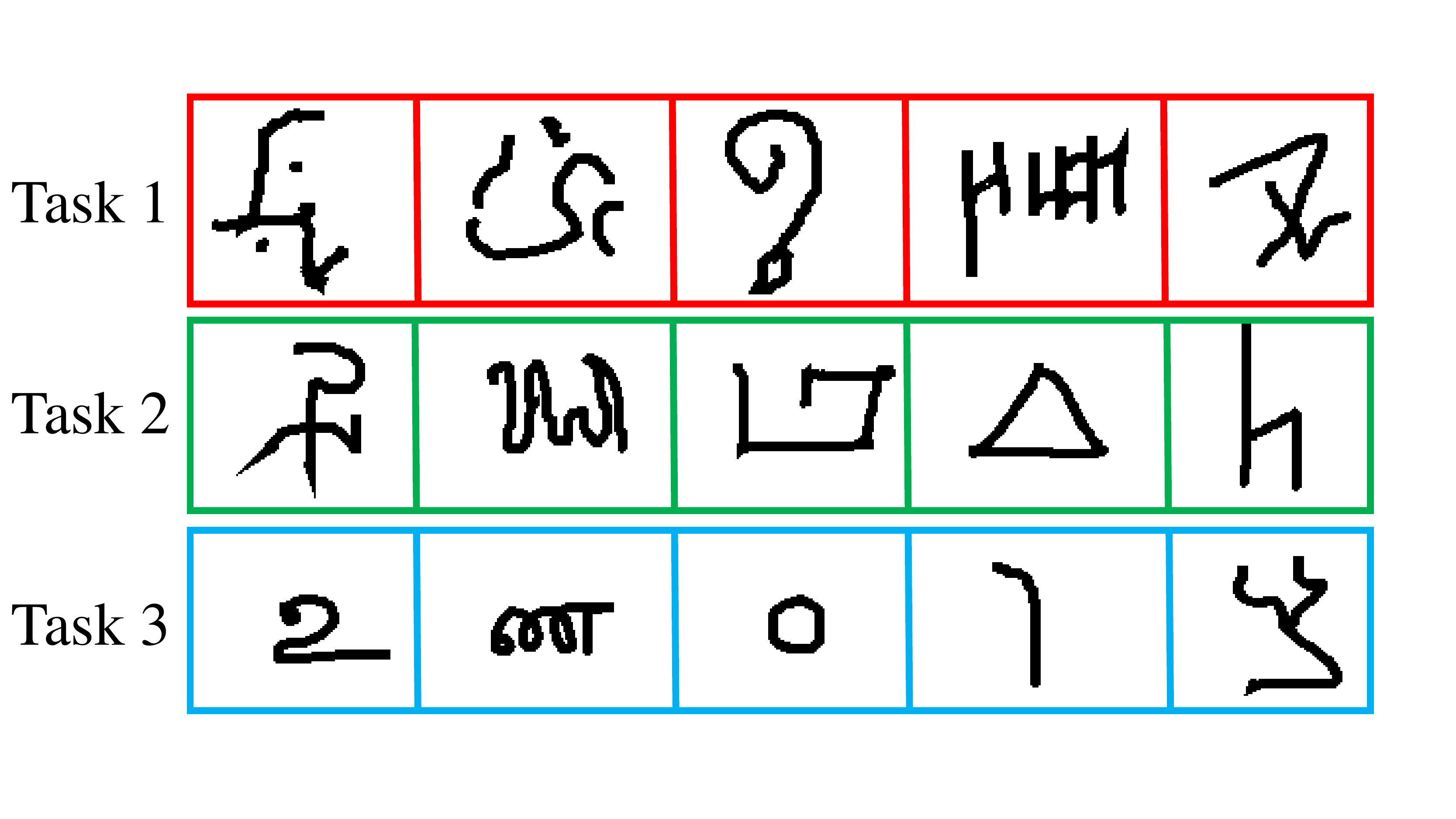}
        \end{minipage}}
    \subfigure[5-way 1-shot tasks generated from MiniImageNet or TieredImageNet dataset]{
        \begin{minipage}[t]{0.48\linewidth}\label{imagenet_tasks}
        %\vspace{-0.4cm}
        \includegraphics[width=0.31\textheight]{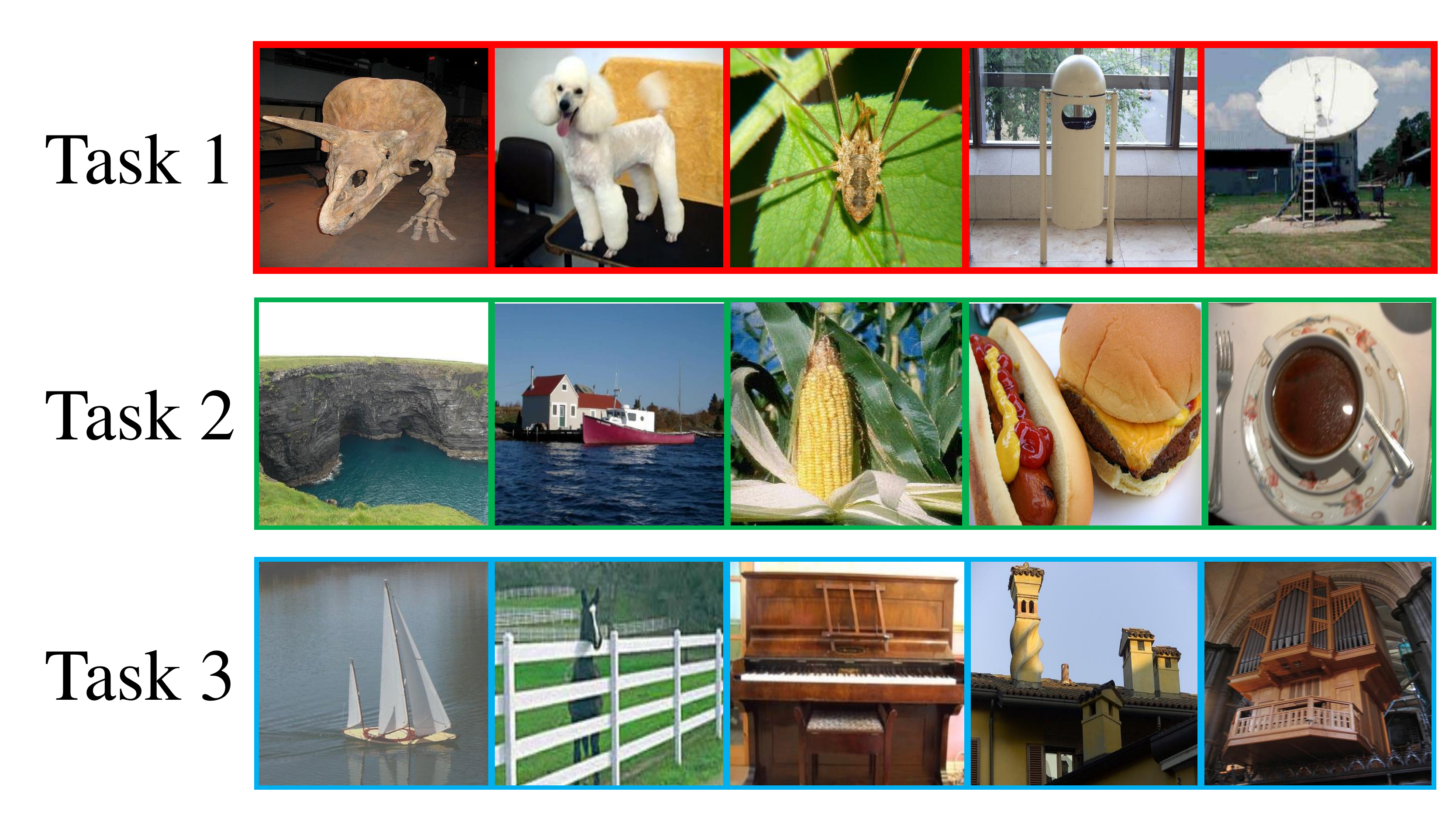}
        \end{minipage}}
\caption{Tasks used in our experiments. (a). Tasks generated from Omniglot. (b). Tasks generated from MiniImageNet or TieredImageNet dataset.}
\label{Fig:DSD_hyper}
\vspace{-2em}
\end{center}
\end{figure}
\begin{table}[h]
\vspace{-0.5em}
\caption{Detailed experimental settings for Omniglot, MiniImageNet, TieredImageNet datasets with DSD-based Reptile.}
\vspace{-1em}
\label{DSD_Detailed_experimental_settings}
\begin{center}
\begin{tabular}{|l|c|c|c|}
\hline
\bf Hyperparameters         &\bf Omniglot       &\bf MiniImageNet       &\bf TieredImageNet\\
\hline
classes                     &$5$                &$5$                    &$5$\\
shot                        &$1$ or $5$         &$1$ or $5$             &$1$ or $5$\\
inner batch                 &$10$               &$10$                   &$6$\\
inner iterations            &$5$                &$8$                    &$8$\\
outer learning rate         &$1$                &$1$                    &$1$\\
meta batch                  &$5$                &$5$                    &$5$\\
meta iterations             &$10^4$             &$10^4$                 &$10^4$\\
evaluation batch            &$5$                &$5$                    &$5$\\
evaluation iterations       &$50$               &$50$                   &$50$\\
inner learning rate         &$0.001$            &$0.001$                &$0.001$\\
pre-train iterations        &$3\times10^4$      &$3\times10^4$          &$3\times10^4$\\
pruning iterations(32c)          &$5\times10^4$      &$5\times10^4$     &$5\times10^4$\\
retrain iterations(32c)          &$2\times10^4$      &$2\times10^4$            &$2\times10^4$\\
pruning iterations(64/128/256c)  &$5\times10^4$      &$6\times10^4$     &$5\times10^4$\\
retrain iterations(64/128/256c)  &$2\times10^4$      &$10^4$            &$2\times10^4$\\
\hline
\end{tabular}
\end{center}
\vspace{-2em}
\end{table}

\begin{table}[h]
%\vspace{-1em}
\caption{Detailed experimental settings for Omniglot, MiniImageNet, TieredImageNet datasets with IHT-based Reptile.}
\vspace{-1em}
\label{IHT_Detailed_experimental_settings}
\begin{center}
\begin{tabular}{|l|c|c|c|}
\hline
\bf Hyperparameters         &\bf Omniglot       &\bf MiniImageNet       &\bf TieredImageNet\\
\hline
classes                     &$5$                &$5$                    &$5$\\
shot                        &$1$ or $5$         &$1$ or $5$             &$1$ or $5$\\
inner batch                 &$10$               &$10$                   &$6$\\
inner iterations            &$5$                &$8$                    &$8$\\
outer learning rate         &$1$                &$1$                    &$1$\\
meta batch                  &$5$                &$5$                    &$5$\\
meta iterations             &$10^4$             &$10^4$                 &$10^4$\\
evaluation batch            &$5$                &$5$                    &$5$\\
evaluation iterations       &$50$               &$50$                   &$50$\\
inner learning rate         &$0.001$            &$0.001$                &$0.001$\\
epoch numbers               &$5$                &$5$                    &$5$\\
iteration numbers per interval          &$2\times10^4$      &$2\times10^4$     &$2\times10^4$\\
pruning iterations          &$1.5\times10^4$     &$1.5\times10^4$         &$1.5\times10^4$\\
retrain iterations          &$5\times10^3$      &$5\times10^3$          &$5\times10^3$\\
\hline
\end{tabular}
\end{center}
\vspace{-2em}
\end{table}

\textbf{Omniglot}\quad The Omniglot dataset has 1623 characters from 50 alphabets. Each character contains 20 instances drawn by different individuals. The size of each image is 28$\times$28. We randomly select 1200 characters for meta training and the rest are used for meta testing. Following~\cite{santoro2016meta}, we also adopt a data augmentation strategy based on image rotation to enhance performance.

\textbf{MiniImageNet}\quad The MiniImageNet dataset consists of 100 classes from the ImageNet dataset~\cite{krizhevsky2012imagenet} and each class contains $600$ images of size $84\times84\times$3. There are $64$ classes used for training, $12$ classes for validation and $24$ classes for testing.

\textbf{TieredImageNet}\quad The TieredImageNet dataset consists of 608 classes from the ILSVRC-12 dataset \cite{russakovsky2015imagenet} and each image is scaled to $84\times84\times3$. There are 351 classes used for training, 97 classes for validation and 160 classes used for testing.

\subsection{Detailed Experimental Settings}\label{subsec:Hyper_settings}
%\subsubsection{Classification Experiments}\label{classification_exp_details}
The experimental details of DSD-based Reptile and IHT-based Reptile can respectively be seen in Table~\ref{DSD_Detailed_experimental_settings} and Table~\ref{IHT_Detailed_experimental_settings}. There are two points of hyperparameter settings that should be highlighted.

\begin{itemize}
\item The outer learning rate has an initial value $1.0$ which will decay with iteration added.
\item For MiniImageNet~\cite{vinyals2016matching} with DSD-based Reptile, the iteration number of pruning phase for $32$-channel case is $5\times10^4$ and for $64/128/256$-channel case is $6\times10^4$. Correspondingly, the iteration number of retraining phase for $32$-channel case is $2\times10^4$ and for $64/128/256$-channel case is $10^4$.
\end{itemize}

%\subsection{Dataset}

%\newpage
\section{Additional Experimental Results}\label{apdx_c}

This appendix contains complete experimental results for Omniglot, MiniImageNet and TieredImageNet datasets. We performed our methods on 4-layer CNNs with varying channel number $\{32, 64, 128, 256\}$ as mentioned in Section \ref{Appdx:Detailed_Settings}.

\subsection{Results on Omniglot dataset}\label{Appdx:Omniglot_complete_results}
\begin{table}[h]
\caption{Few Shot Classification results on Omniglot dataset for 4-layer convolutional network with different channels on 5-way 1-shot and 5-way 5-shot tasks. The ``$\pm$'' shows $95\%$ confidence intervals over tasks. The evaluation  baselines are run by us.}
\label{Omniglot5wayresults_table}
\begin{center}
\begin{tabular}{|l|l|l|cc|}
\hline
\bf Methods     &\bf Backbone      &\bf Rate    &\bf 5-way 1-shot    &\bf 5-way 5-shot\\
\hline
\multirow{4}*{Reptile baseline}     &32-32-32-32            &0$\%$      &96.63$\pm$0.17$\%$     &99.31$\pm$0.07$\%$ \\
                                   ~&64-64-64-64            &0$\%$      &97.68$\pm$0.10$\%$     &99.48$\pm$0.06$\%$ \\
                                   ~&128-128-128-128        &0$\%$      &97.99$\pm$0.11$\%$     &99.60$\pm$0.13$\%$ \\
                                   ~&256-256-256-256        &0$\%$      &98.05$\pm$0.13$\%$     &99.65$\pm$0.06$\%$ \\ \hline

\multirow{16}*{DSD-based Reptile}   &\multirow{4}*{32-32-32-32}        &$10\%$      &96.42$\pm$0.17$\%$      &\bf99.38$\pm$0.07$\%$ \\
                                                                  ~&&$20\%$      &95.98$\pm$0.18$\%$      &99.33$\pm$0.07$\%$ \\
                                                                  ~&&$30\%$      &96.22$\pm$0.17$\%$      &99.23$\pm$0.08$\%$ \\
                                                                  ~&&$40\%$      &96.53$\pm$0.17$\%$      &99.37$\pm$0.07$\%$ \\ \cline{2-5}

                                   ~&\multirow{4}*{64-64-64-64}        &$10\%$      &\bf97.64$\pm$0.02$\%$   &\bf 99.50$\pm$0.05$\%$ \\
                                                                  ~&&$20\%$      &97.60$\pm$0.07$\%$      &99.49$\pm$0.04$\%$ \\
                                                                  ~&&$30\%$      &97.47$\pm$0.05$\%$      &99.49$\pm$0.05$\%$ \\
                                                                  ~&&$40\%$      &97.43$\pm$0.01$\%$      &99.45$\pm$0.03$\%$ \\ \cline{2-5}

                                   ~&\multirow{4}*{128-128-128-128}    &$10\%$      &\bf 98.04$\pm$0.10$\%$  &99.61$\pm$0.10$\%$ \\
                                                                  ~&&$20\%$      &97.99$\pm$0.10$\%$      &99.62$\pm$0.12$\%$ \\
                                                                  ~&&$30\%$      &97.96$\pm$0.12$\%$      &\bf 99.63$\pm$0.12$\%$ \\
                                                                  ~&&$40\%$      &97.99$\pm$0.10$\%$      &99.61$\pm$0.10$\%$ \\ \cline{2-5}

                                   ~&\multirow{4}*{256-256-256-256}        &$10\%$      &\bf98.12$\pm$0.12$\%$   &\bf99.68$\pm$0.05$\%$ \\
                                                                  ~&&$20\%$      &98.02$\pm$0.13$\%$      &99.66$\pm$0.05$\%$ \\
                                                                  ~&&$30\%$      &97.96$\pm$0.13$\%$      &99.67$\pm$0.05$\%$ \\
                                                                  ~&&$40\%$      &97.99$\pm$0.10$\%$      &99.63$\pm$0.06$\%$ \\ \hline
\multirow{16}*{IHT-based Reptile}      &\multirow{4}*{32-32-32-32}        &$10\%$      &\bf96.65$\pm$0.16$\%$   &99.49$\pm$0.06$\%$ \\
                                                                  ~&&$20\%$      &96.54$\pm$0.17$\%$      &\bf99.57$\pm$0.06$\%$ \\
                                                                  ~&&$30\%$      &96.45$\pm$0.17$\%$      &99.52$\pm$0.06$\%$ \\
                                                                  ~&&$40\%$      &96.21$\pm$0.18$\%$      &99.48$\pm$0.07$\%$ \\ \cline{2-5}

                               ~&\multirow{4}*{64-64-64-64}        &$10\%$      &97.63$\pm$0.14$\%$      &99.49$\pm$0.06$\%$ \\
                                                                  ~&&$20\%$      &97.60$\pm$0.13$\%$      &\bf99.57$\pm$0.06$\%$ \\
                                                                  ~&&$30\%$      &\bf97.77$\pm$0.15$\%$   &99.52$\pm$0.06$\%$ \\
                                                                  ~&&$40\%$      &97.51$\pm$0.1$\%$       &99.48$\pm$0.07$\%$ \\ \cline{2-5}

                               ~&\multirow{4}*{128-128-128-128}        &$10\%$      &98.12$\pm$0.12$\%$      &99.63$\pm$0.06$\%$ \\
                                                                  ~&&$20\%$      &\bf98.22$\pm$0.12$\%$   &99.64$\pm$0.05$\%$ \\
                                                                  ~&&$30\%$      &98.01$\pm$0.13$\%$      &\bf 99.65$\pm$0.05$\%$ \\
                                                                  ~&&$40\%$      &98.06$\pm$0.12$\%$      &99.63$\pm$0.06$\%$ \\ \cline{2-5}

                               ~&\multirow{4}*{256-256-256-256}        &$10\%$      &\bf98.16$\pm$0.12$\%$   &99.66$\pm$0.05$\%$ \\
                                                                  ~&&$20\%$      &98.08$\pm$0.13$\%$      &\bf99.69$\pm$0.05$\%$ \\
                                                                  ~&&$30\%$      &98.05$\pm$0.13$\%$      &99.64$\pm$0.05$\%$ \\
                                                                  ~&&$40\%$      &97.90$\pm$0.13$\%$      &99.65$\pm$0.05$\%$ \\
\hline
\end{tabular}
\end{center}
\end{table}
 The baselines and all the results of Omniglot dataset are reported in Table \ref{Omniglot5wayresults_table}. For each case, both DSD-based Reptile approach and IHT-based Reptile approach are evaluated on various pruning rates. The settings are the same as proposed in Section \ref{subsec:Hyper_settings}.

For 32-channel case and 64-channel cases, which is less prone to be overfitting, both DSD-based Reptile approach and IHT-based Reptile approach tend to achieve comparable performance to baselines. %On 5-way 1-shot task settings, it can be observed that the accuracy doesn't drop too much and can achieve similar performance as baselines. The accuracy of the worst result is the case of 32-channel CNNs with 20$\%$ pruning rate which drops about 0.65$\%$, and the best one is the case of 64-channel CNNs with 30$\%$ pruning rate that improves about 0.1$\%$. On 5-way 5-shot task settings, most cases outperform the baseline. As shown in the table, even in 32-channel case, when the pruning rate is 20$\%$, the performance of IHT-based Reptile imporves about 0.26$\%$. For 128-channel and 256-channel cases, our method consistently achieves better performance than the baselines. The most important is that with channels added, the baselines keep increasing, but our method can still improve the performance, which can manifest the impressive advantage of our method.
When the channel size increases to $128$ and $256$, slightly improved performance can be observed. This is consistent with our analysis that overfiting is more likely to happen when channel number is relatively large and weight pruning helps alleviate such phenomenon to improve the generalization performance, which then leads to accuracy improvement with retraining operation.

\subsection{Results on MiniImageNet dataset}\label{Appdx:MiniImageNet_complete_results}

In this section, we report the detailed results of experiments on MiniImageNet dataset.
\begin{table}[h]
\caption{Few Shot Classification results on MiniImageNet dataset for 4-layer convolutional network with different channels on 5 way setting. The ``$\pm$'' shows $95\%$ confidence intervals over tasks. The evaluation  baselines are run by us.}
\label{MiniImageNet5wayresults_table}
\begin{center}
\begin{tabular}{|l|l|l|cc|}
\hline
\bf Methods     &\bf Backbone       &\bf Rate    &\bf 5-way 1-shot    &\bf 5-way 5-shot
\\ \hline
\multirow{4}*{Reptile baseline}     &32-32-32-32        &0$\%$      &50.30$\pm$0.40$\%$     &64.27$\pm$0.44$\%$ \\
                                   ~&64-64-64-64        &0$\%$      &51.08$\pm$0.44$\%$     &65.46$\pm$0.43$\%$ \\
                                   ~&128-128-128-128    &0$\%$      &49.96$\pm$0.45$\%$     &64.40$\pm$0.43$\%$ \\
                                   ~&256-256-256-256    &0$\%$      &48.60$\pm$0.44$\%$     &63.24$\pm$0.43$\%$ \\
\multirow{4}*{CAVIA baseline}       &32-32-32-32        &0$\%$      &47.24$\pm$0.65$\%$     &59.05$\pm$0.54$\%$ \\
                                   ~&128-128-128-128    &0$\%$      &49.84$\pm$0.68$\%$     &64.63$\pm$0.54$\%$ \\
                                   ~&512-512-512-512    &0$\%$      &51.82$\pm$0.65$\%$     &65.85$\pm$0.55$\%$ \\
\hline

\multirow{13}*{DSD-based Reptile}         &\multirow{4}*{32-32-32-32}     &$10\%$     &\bf50.65$\pm$0.45$\%$  &\bf65.29$\pm$0.44$\%$ \\
                                                                        ~&&$20\%$     &49.94$\pm$0.43$\%$     &64.65$\pm$0.43$\%$ \\
                                                                        ~&&$30\%$     &50.18$\pm$0.43$\%$     &\bf65.78$\pm$0.41$\%$ \\
                                                                        ~&&$40\%$     &\bf50.83$\pm$0.45$\%$  &\bf65.24$\pm$0.44$\%$ \\ \cline{2-5}

                                   ~&\multirow{4}*{64-64-64-64}           &$10\%$     &51.12$\pm$0.45$\%$     &65.80$\pm$0.44$\%$ \\
                                                                        ~&&$20\%$     &\bf51.91$\pm$0.45$\%$  &\bf67.21$\pm$0.43$\%$ \\
                                                                        ~&&$30\%$     &\bf51.91$\pm$0.45$\%$  &\bf67.23$\pm$0.43$\%$ \\
                                                                        ~&&$40\%$     &\bf51.96$\pm$0.45$\%$  &\bf67.17$\pm$0.43$\%$ \\ \cline{2-5}

                                   ~&\multirow{4}*{128-128-128-128}       &$30\%$     &51.98$\pm$0.45$\%$     &68.16$\pm$0.43$\%$ \\
                                                                        ~&&$40\%$     &52.15$\pm$0.45$\%$     &68.19$\pm$0.43$\%$ \\
                                                                        ~&&$50\%$     &52.08$\pm$0.45$\%$     &\bf68.87$\pm$0.42$\%$ \\
                                                                        ~&&$60\%$     &\bf52.27$\pm$0.45$\%$  &68.44$\pm$0.42$\%$ \\ \cline{2-5}
                                   ~&256-256-256-256                      &$60\%$     &\bf53.00$\pm$0.45$\%$  &\bf68.04$\pm$0.42$\%$ \\ \hline

\multirow{13}*{IHT-based Reptile}         &\multirow{4}*{32-32-32-32}     &$10\%$     &\bf50.45$\pm$0.45$\%$  &63.91$\pm$0.46$\%$ \\
                                                                        ~&&$20\%$     &50.26$\pm$0.47$\%$     &63.63$\pm$0.45$\%$ \\
                                                                        ~&&$30\%$     &50.21$\pm$0.44$\%$     &\bf65.05$\pm$0.45$\%$ \\
                                                                        ~&&$40\%$     &49.74$\pm$0.46$\%$     &64.15$\pm$0.45$\%$ \\ \cline{2-5}

                                   ~&\multirow{4}*{64-64-64-64}           &$10\%$     &\bf52.23$\pm$0.45$\%$  &\bf66.08$\pm$0.43$\%$ \\
                                                                        ~&&$20\%$     &\bf52.13$\pm$0.46$\%$  &\bf66.78$\pm$0.43$\%$ \\
                                                                        ~&&$30\%$     &51.98$\pm$0.45$\%$     &66.14$\pm$0.43$\%$ \\
                                                                        ~&&$40\%$     &\bf52.59$\pm$0.45$\%$  &\bf67.41$\pm$0.43$\%$ \\ \cline{2-5}

                                   ~&\multirow{4}*{128-128-128-128}       &$30\%$     &51.64$\pm$0.45$\%$     &67.05$\pm$0.43$\%$ \\
                                                                        ~&&$40\%$     &52.73$\pm$0.45$\%$     &\bf68.69$\pm$0.42$\%$ \\
                                                                        ~&&$50\%$     &52.76$\pm$0.45$\%$     &67.63$\pm$0.43$\%$ \\
                                                                        ~&&$60\%$     &\bf52.95$\pm$0.45$\%$  &68.04$\pm$0.42$\%$ \\ \cline{2-5}

                                   ~&256-256-256-256                      &$60\%$     &\bf49.85$\pm$0.44$\%$  &\bf66.56$\pm$0.42$\%$ \\ \hline
\end{tabular}
\end{center}
\end{table}

From the table, it can be obviously observed that our method achieves remarkable performance consistently. For one thing, with the number of channels increasing, the accuracies of our methods keep being improved while the baselines perform oppositely.
For example, in the 32-channel setting in which the model is less prone to overfit, when applying DSD-based Reptile with $10\%$ and $40\%$ pruning rate, the accuracy gain is $0.35\%$ and $0.5\%$ on 5-way 1-shot tasks and $1.02\%$ and $1\%$ on 5-way 5-shot tasks.
In the 64-channel setting, DSD-based Reptile respectively achieves $0.83\%$, $0.83\%$, $0.88\%$ improvements over 5-way 1-shot baseline and $1.75\%$, $1.77\%$, $1.18\%$ improvements over 5-way 5-shot baseline with pruning rates $20\%$, $30\%$, $40\%$. Meanwhile our IHT-based Reptile approach respectively improves about $1.15\%$, $1.05\%$, $1.51\%$ on 5-way 1-shot tasks and $0.62\%$, $1.32\%$ and $1.95\%$ on 5-way 5-shot tasks with pruning rates $10\%$, $20\%$, $40\%$.
In the setting of 128-channel, all the cases of our method outperform the baseline remarkably, and the best accuracy of DSD-based Reptile on 5-way 1-shot tasks is nearly $3\%$ higher than the baseline while on 5-way 5-shot tasks the gain is about $4.47\%$.

Our method also outperforms CAVIA~\cite{zintgraf2019fast}, which can increase the network size without overfitting. With our method, CNNs with 64 channels can obtain better performance than the best result of CAVIA. %To evaluate our method on more complex networks, we also perform the experiments on MetaOptNet~\cite{lee2019meta}. We can observe that our DSD variant algorithm achieves $0.2\%$ improvement on 5-way 1-shot tasks and similar performance on 5-way 5-shot tasks. However, there is a trade-off between accuracy and training time because in our experiments the total training epoch number is 40 while it is 21 in ~\cite{lee2019meta}.%These results demonstrate that by embedding network pruning approaches into meta-learning, the pruning phase can help to avoid learning the useless and thus ease the overfitting, and the retraining phase which preserves the capacity of the model assist to improve the performance of the generalization.

\subsection{Results on TieredImageNet dataset}\label{Appdx:TieredImageNet_complete_results}
\begin{table}[h]
\caption{Few Shot Classification results on TieredImageNet dataset for 4-layer convolutional network with different channels on 5 way setting. The ``$\pm$'' shows $95\%$ confidence intervals over tasks. The evaluation  baselines are run by us.}
\label{TieredImageNet5wayresults_table}
\begin{center}
\begin{tabular}{|l|l|l|cc|}
\hline
\bf Methods     &\bf Backbone       &\bf Rate    &\bf 5-way 1-shot    &\bf 5-way 5-shot
\\ \hline
\multirow{4}*{Reptile baseline}     &32-32-32-32        &0$\%$      &50.52$\pm$0.45$\%$     &64.63$\pm$0.44$\%$ \\
                                   ~&64-64-64-64        &0$\%$      &51.98$\pm$0.45$\%$     &67.70$\pm$0.43$\%$ \\
                                   ~&128-128-128-128    &0$\%$      &53.30$\pm$0.45$\%$     &69.29$\pm$0.42$\%$ \\
                                   ~&256-256-256-256    &0$\%$      &54.62$\pm$0.45$\%$     &68.06$\pm$0.42$\%$ \\ \hline

\multirow{8}*{DSD-based Reptile}          &\multirow{2}*{32-32-32-32}     &$10\%$     &\bf 50.94$\pm$0.46$\%$     &64.65$\pm$0.44$\%$ \\
                                                                  ~&&$20\%$     &49.85$\pm$0.46$\%$         &63.72$\pm$0.44$\%$ \\ \cline{2-5}
                                   ~&\multirow{2}*{64-64-64-64}     &$10\%$      &\bf 52.62$\pm$0.46$\%$     &66.69$\pm$0.43$\%$ \\
                                                                  ~&&$20\%$     &51.95$\pm$0.45$\%$         &66.05$\pm$0.43$\%$ \\ \cline{2-5}
                                   ~&\multirow{2}*{128-128-128-128} &$10\%$     &53.39$\pm$0.46$\%$         &67.22$\pm$0.43$\%$ \\
                                                                  ~&&$20\%$     &52.61$\pm$0.46$\%$         &66.39$\pm$0.43$\%$ \\ \cline{2-5}
                                   ~&\multirow{2}*{256-256-256-256} &$10\%$     &54.55$\pm$0.45$\%$         &\bf68.60$\pm$0.43$\%$ \\
                                                                  ~&&$20\%$     &\bf 54.98$\pm$0.45$\%$     &67.98$\pm$0.43$\%$ \\ \hline

\multirow{8}*{IHT-based Reptile}         ~&\multirow{2}*{32-32-32-32}     &$10\%$     &\bf 50.58$\pm$0.46$\%$     &63.09$\pm$0.45$\%$ \\
                                                                  ~&&$20\%$     &50.19$\pm$0.46$\%$         &63.42$\pm$0.44$\%$ \\ \cline{2-5}
                                   ~&\multirow{2}*{64-64-64-64}     &$10\%$     &51.75$\pm$0.45$\%$         &65.20$\pm$0.44$\%$ \\
                                                                  ~&&$20\%$     &\bf53.22$\pm$0.46$\%$      &66.15$\pm$0.44$\%$ \\ \cline{2-5}
                                   ~&\multirow{2}*{128-128-128-128} &$10\%$     &\bf 53.48$\pm$0.45$\%$     &69.36$\pm$0.42$\%$ \\
                                                                  ~&&$20\%$     &52.98$\pm$0.45$\%$         &66.22$\pm$0.43$\%$ \\ \cline{2-5}
                                   ~&\multirow{2}*{256-256-256-256} &$10\%$     &\bf 55.06$\pm$0.45$\%$     &67.60$\pm$0.43$\%$ \\
                                                                  ~&&$20\%$     &54.38$\pm$0.45$\%$         &\bf69.36$\pm$0.42$\%$ \\
\hline
\end{tabular}
\end{center}
\end{table}

In this section, we present the detailed results of experiments on TieredImageNet dataset in Table~\ref{TieredImageNet5wayresults_table}.

From the table, we can observe that our method achieves good performance on 5-way 1-shot classification tasks. For example, in 32-channel settings, the accuracy of DSD-based Reptile with $10\%$ pruning rate is $\sim 0.5\%$ higher than baseline; in 64-channel settings, both DSD-based Reptile and IHT-based Reptile improve the performance evidently, respectively are $0.64\%$ and $1.24\%$; and in 256-channel settings, the best performance achieves $0.44\%$ improvement over the baseline.

However, in most 5-way 5-shot classification tasks, the performance of our method drops. We conjecture that the reason is that TieredImageNet dataset, compared with MiniImageNet dataset, contains more classes from which the networks can learn more prior knowledge and thus ease the overfitting.
%Although our method achieves success on 5-way 1-shot tasks, it doesn't perform well on 5-way 5-shot tasks. We conjecture that the reasons that our method fails are that for one thing, compared with 5-way 1-shot tasks, 5-way 5-shot tasks contain more data and are less possible to overfit the data; for another thing, TieredImageNet dataset includes more classes and an specific image class tends to not repeatedly appear in tasks, which also reduce the possibility of overfitting.

\end{document}